\documentclass{article}

\usepackage[final]{neurips_2021}

\usepackage[utf8]{inputenc} 
\usepackage[T1]{fontenc}    
\usepackage{hyperref}       
\usepackage{url}            
\usepackage{booktabs}       
\usepackage{amsfonts}       
\usepackage{nicefrac}       
\usepackage{microtype}      
\usepackage{xcolor}         

\usepackage{algorithmic}
\usepackage[utf8]{inputenc}
\usepackage{amsfonts,amssymb,amsmath,amsthm}
\usepackage{algorithm}
\usepackage{enumitem}
\usepackage{times}
\usepackage[italic,defaultmathsizes]{mathastext}
\usepackage{delimset}
\usepackage{nicefrac}
\usepackage{thmtools}
\usepackage{thm-restate}
\usepackage{hyperref}
\usepackage[capitalise]{cleveref}

\usepackage{xcolor}

\usepackage{mathtools}
\DeclarePairedDelimiter\br{(}{)}
\DeclarePairedDelimiter\brs{[}{]}
\DeclarePairedDelimiter\brc{\{}{\}}

\newcommand{\E}{\mathbb{E}}

\newcommand{\G}{\mathbb{G}}

\newcommand{\VAR}{\mathrm{Var}}

\usepackage{accents}

\newtheorem{theorem}{Theorem}[section]
\newtheorem{lemma}[theorem]{Lemma}

\newtheorem*{lemma*}{Lemma}
\newtheorem*{theorem*}{Theorem}

\theoremstyle{definition}

\newtheorem{definition}{Definition}

\newtheorem{remark}{Remark}

\newcommand{\regret}{R_K}

\newcommand{\bbE}{\mathbb{E}}
\newcommand{\bbR}{\mathbb{R}}

\newcommand{\calM}{\mathcal{M}}
\newcommand{\calR}{\mathcal{R}}
\newcommand{\calA}{\mathcal{A}}
\newcommand{\calB}{\mathcal{B}}

\newcommand{\wt}{\widetilde}
\newcommand{\wh}{\widehat}
\newcommand{\tO}{\wt O}

\renewcommand{\Pr}{\mathbb{P}}
\newcommand{\ind}{\mathbb{I}}
\newcommand{\TV}[2]{\text{TV}(#1, #2)}
\newcommand{\KL}[2]{\text{KL}(#1 \;\|\; #2)}

\newcommand{\sinit}{s_\text{init}}
\newcommand{\ssink}{g}
\newcommand{\ctg}[1]{J^{#1}}
\newcommand{\ctgopt}{\ctg{\piopt}}

\newcommand{\hatctg}[1]{\wh{J}^{#1}}
\newcommand{\costbound}{B_\star}

\newcommand{\timebound}{T_\star}
\newcommand{\piopt}{\pi^\star}
\newcommand{\policytime}[1]{T^{#1}}

\newcommand{\timeopt}{\policytime{\piopt}}
\newcommand{\propset}{\Pi_\text{proper}}
\newcommand{\geventi}[1]{G^{#1}}

\newcommand{\indevent}[1]{\ind \{ #1 \}}
\newcommand{\numintervals}{M}

\newcommand{\cmin}{c_{\text{min}}}

\newcommand{\Ik}{I^k}
\newcommand{\Ikk}[1]{I^{#1}}

\newcommand{\knownthresh}{\omega_{\calA}}

\newcommand{\traj}[1]{U^{#1}}
\newcommand{\trajconcat}[1]{\bar U^{#1}}

\DeclareMathOperator*{\argmin}{arg\,min}

\title{Minimax Regret for Stochastic Shortest Path}

\author{%
  Alon Cohen \\
  Tel-Aviv University and Google Research, Tel Aviv\\
  \texttt{aloncohen@google.com} \\
  \And
  Yonathan Efroni \\
  Microsoft Research, New York\\
  \texttt{jonathan.efroni@gmail.com} \\
  \And
  Yishay Mansour \\
  Tel-Aviv University and Google Research, Tel Aviv\\
  \texttt{mansour@tau.ac.il} \\
  \And
  Aviv Rosenberg \\
  Tel-Aviv University\\
  \texttt{avivros007@gmail.com} \\
}

\begin{document}

\maketitle

\begin{abstract}
    We study the Stochastic Shortest Path (SSP) problem in which an agent has to reach a goal state in minimum total expected cost. 
    In the learning formulation of the problem, the agent has no prior knowledge about the costs and dynamics of the model. 
    She repeatedly interacts with the model for $K$ episodes, and has to minimize her regret.
    In this work we show that the minimax regret for this setting is $\widetilde O(\sqrt{ (B_\star^2 + B_\star) |S| |A| K})$ where $B_\star$ is a bound on the expected cost of the optimal policy from any state, $S$ is the state space, and $A$ is the action space.
    This matches the $\Omega (\sqrt{ B_\star^2 |S| |A| K})$ lower bound of \citet{rosenberg2020near} for $B_\star \ge 1$, and improves their regret bound by a factor of $\sqrt{|S|}$.
    For $B_\star < 1$ we prove a matching lower bound of $\Omega (\sqrt{ B_\star |S| |A| K})$.
    Our algorithm is based on a novel reduction from SSP to finite-horizon MDPs. 
    To that end, we provide an algorithm for the finite-horizon setting whose leading term in the regret depends polynomially on the expected cost of the optimal policy and only logarithmically on the horizon.
\end{abstract}

\section{Introduction}

We study the stochastic shortest path (SSP) problem in which an agent aims to reach a predefined goal state while minimizing her total expected cost.
This is one of the most basic models of reinforcement learning (RL) that includes both finite-horizon and discounted Markov Decision Processes (MDPs) as special cases. In addition, SSP captures a wide variety of realistic scenarios such as car navigation, game playing and drone flying.

We study an online version of SSP in which both the immediate costs and transition distributions of the model are initially unknown to the agent. The agent interacts with the model for $K$ episodes, in each of which she attempts to reach the goal state with minimal cumulative cost. 
A main challenge in the online model is found when instantaneous costs are small. For example, any learning algorithm that attempts to myopically minimize the accumulated costs might get caught in a cycle with zero cost and never reach the goal state. Nonetheless, even if the costs are not zero, only very small, the agent must be able to trade off the need to minimize costs with that of reaching the goal quickly. 

The online setting was originally suggested by \cite{tarbouriech2019noregret} who gave an algorithm with $\tO(K^{2/3})$ regret guarantee. In a follow-up work, \cite{rosenberg2020near} improved the previous bound to $\tO(\costbound |S| \sqrt{|A| K})$, where $S$ is the state space, $A$ is the action space, and $\costbound$ is an upper bound on the total expected cost of the optimal policy when initialized at any state. \cite{rosenberg2020near} also provide a lower bound of $\Omega(\costbound \sqrt{|S| |A| K})$ -- leaving a gap of $\sqrt{|S|}$ between the upper and lower bounds. 
In this work, unlike the previously mentioned works that assume the cost function is deterministic and known, we consider the case where the costs are i.i.d.\ and initially unknown. We prove upper and lower bounds for this case, proving that the optimal regret is of order $\wt \Theta(\sqrt{(\costbound^2 + \costbound) |S| |A| K})$.

The algorithms of both \cite{tarbouriech2019noregret,rosenberg2020near} were based on a direct application of the ``Optimism in the Face of Uncertainty'' principle to the SSP model, following the ideas behind the UCRL2 algorithm \citep{jaksch2010near} for average-reward MDPs.
In this work we take a different approach. 
We propose a novel black-box reduction to finite-horizon MDPs, showing that the SSP problem is not harder than the finite-horizon setting assuming prior knowledge on the expected time it takes for the optimal policy to reach the goal state.
While the reduction itself is simple, the analysis is highly nontrivial as one has to show that the goal state is indeed reached in every episode without incurring excessive costs in the process.

The idea of reducing SSP to finite-horizon was previously used by \citet{chen2020minimax,chen2021finding} for SSP with adversarially changing costs.
However, they run one finite-horizon episode in every SSP episode and then simply try to reach the goal as fast as possible, while we restart a new finite-horizon episode every $H$ steps.
This modification is what enables us to obtain the optimal and improved dependence in the number of states.

In addition, we provide a new algorithm for regret minimization in finite-horizon MDPs called \verb|ULCVI|.
We show that (for large enough number of episodes) its regret depends polynomially on the expected cost of the optimal policy $B_\star$, and only logarithmically on the horizon length $H$.
This implies that the correct measure for the regret is the expected cost of the optimal policy and not the length of the horizon. 
We note that regret with logarithmic dependence in the horizon $H$ was also obtained by \citet{zhang2020reinforcement}, yet they make a much stronger assumption: that the cumulative cost of \emph{every} trajectory is bounded by $1$.
In contrast, we only assume that the \emph{expected} cost of the optimal policy is bounded by some constant $\costbound$, while other policies may suffer a cost of $H$.

Our reduction, when combined with our finite-horizon algorithm \verb|ULCVI|, guarantees SSP regret of $\wt O( \sqrt{ (\costbound^2 + \costbound) |S| |A| K} )$.
This matches the lower bound of \citet{rosenberg2020near} for $\costbound \ge 1$ up to logarithmic factors.
However, their lower bound does not hold for $\costbound < 1$ suggesting that this is not the correct rate in this case.
Indeed, we prove a tighter lower bound of $\Omega(\sqrt{\costbound |S| |A| K})$ for $\costbound < 1$, showing that our regret guarantees are minimax optimal in all cases. 

As a final remark we note that, following our work, \citet{tarbouriech2021stochastic} were able to obtain a comparable regret bound for SSP without prior knowledge of the optimal policy's expected time to reach the goal state.

\subsection{Additional related work}

\textbf{Planning for stochastic shortest path.}
Early work by \cite{bertsekas1991analysis} studied planning in SSPs, i.e., computing the optimal strategy efficiently when parameters are known. 
Under certain assumptions, they established that the optimal strategy is a deterministic stationary policy and can be computed efficiently using standard planning algorithms, e.g., Value Iteration and LP.

\textbf{Adversarial stochastic shortest path.}
\citet{rosenberg2020stochastic} presented stochastic shortest path with adversarially changing costs.
Their regret bounds were improved by \citet{chen2020minimax,chen2021finding} using a reduction to online loop-free SSP (see next paragraph).
As mentioned before, our reduction is different and therefore able to remove the extra $\sqrt{|S|}$ factor in the regret.

\textbf{Regret minimization in MDPs.}
There is a vast literature on regret minimization in RL that mostly builds on the optimism principle.
Most literature focuses on the tabular setting \citep{jaksch2010near,azar2017minimax,jin2018q,fruit2018efficient,zanette2019tighter,efroni2019tight,simchowitz2019non}, but recently it was extended to function approximation under various assumptions \citep{yang2019sample,jin2020provably,zanette2020frequentist,zanette2020learning}.

\textbf{Online loop-free SSP.}
A different line of work considers finite-horizon MDPs with adversarially changing costs \citep{neu2010loopfree,neu2012adversarial,zimin2013online,rosenberg2019full,rosenberg2019bandit,jin2019learning,cai2019provably,efroni2020optimistic,lancewicki2020learning,lee2020bias,jin2020simultaneously}.
They refer to finite-horizon adversarial MDPs as online loop-free SSP.
This is not to be confused with our setting in which the interaction between the agent and the environment ends only when (and if) the goal state is reached, and not after a fixed number of steps $H$.
See \citet{rosenberg2020stochastic,chen2020minimax} for a discussion on the differences between the models.

\section{Preliminaries and main results} \label{sec:preliminaries}

An instance of the SSP problem is defined by an MDP $\calM = (S,A,P,c,\sinit,\ssink)$ where $S$ is a finite state space and $A$ is a finite action space. 
The agent begins at an initial state $\sinit \in S$, and ends her interaction with $\calM$ by arriving at the goal state $\ssink$ (where $\ssink\not\in S$).
Whenever she plays action $a$ in state $s$, she pays a cost $C \in [0,1]$ drawn i.i.d. from a distribution with expectation $c(s,a) \in [0,1]$ and the next state $s' \in S \cup \{ \ssink \}$ is chosen with probability $P(s' \mid s,a)$. 
Note that the transition function $P$ satisfies $\sum_{s' \in S \cup \{ \ssink \}} P(s' \mid s,a) = 1$ for every $(s,a) \in S \times A$.

\textbf{Proper policies.}
A stationary and deterministic policy $\pi : S \mapsto A$ is a mapping that selects action $\pi(s)$ whenever the agent is at state $s$.
A policy $\pi$ is called \emph{proper} if playing according to $\pi$ ensures that the goal state is reached with probability $1$ when starting from any state (otherwise it is \emph{improper}).
In SSP, the agent has two goals: (a) reach the goal state; (b) minimize the total expected cost.
To facilitate the first goal, we make the basic assumption that there exists at least one proper policy.
In particular, the goal state is reachable from every state, which is clearly a necessary assumption.

Any policy $\pi$ induces a \emph{cost-to-go function} $\ctg{\pi} : S \mapsto [0, \infty]$.
The cost-to-go at state $s$ is defined by
$
    \ctg{\pi}(s) = \lim_{T \rightarrow \infty} \bbE_\pi \brk[s]1{\sum_{t=1}^T c(s_t,a_t) \mid \sinit = s},
$
where the expectation is taken w.r.t the random sequence of states generated by playing according to $\pi$ when the initial state is $s$. For a proper policy $\pi$, it follows that $\ctg{\pi}(s)$ is finite for all $s \in S$. However, note that $\ctg{\pi}(s)$ may be finite even if $\pi$ is improper. We additionally denote by $\policytime{\pi}(s)$ the expected time it takes for $\pi$ to reach $\ssink$ starting at state $s$; in particular, if $\pi$ is proper then $\policytime{\pi}(s)$ is finite for all $s \in S$, and if $\pi$ is improper there must exist some state $s$ such that $\policytime{\pi}(s) = \infty$.

\textbf{Learning formulation.}
Here, the agent does not have any prior knowledge of the cost function $c$ or transition function $P$.
She interacts with the model in episodes: each episode starts at the fixed initial state $\sinit$,\footnote{The initial state is fixed for simplicity of presentation, but it can be chosen adversarially at the beginning of every episode. Without any change to the algorithm or analysis, the same guarantees hold.} and ends when the agent reaches the goal state $\ssink$ (note that she might \emph{never} reach the goal state). 
Success is measured by the agent's regret over $K$ such episodes, that is the difference between her total cost over the $K$ episodes and the total expected cost of the optimal proper policy:
\[
    \regret 
    = 
    \sum_{k=1}^K \sum_{i=1}^{\Ik} C^k_i - K \cdot \min_{\pi \in \propset} \ctg{\pi} (\sinit),
\]
where $\Ik$ is the time it takes the agent to complete episode $k$ (which may be infinite), $C^k_i$ is the cost suffered in the $i$-th step of episode $k$ when the agent visited state-action pair $(s^k_i,a^k_i)$, and
$\propset$ is the set of all stationary, deterministic and proper policies (that is not empty by assumption). 
In the case that $\Ik$ is infinite for some $k$, we define $\regret = \infty$.

We denote the optimal proper policy by $\piopt$, $\ctgopt(\sinit) = \argmin_{\pi \in \propset} \ctg{\pi}(\sinit)$. 
Moreover, let $\costbound > 0$ be an upper bound on the values of $\ctgopt$ and let $\timebound > 0$ be an upper bound on the times $\timeopt$, i.e., $\costbound \geq \max_{s \in S} \ctgopt (s)$ and $\timebound \geq \max_{s \in S} \timeopt (s)$.
Finally, let $D = \max_{s \in S} \min_{\pi \in \propset} \policytime{\pi}(s)$ be the SSP-diameter, and note that $\costbound \le D \le \timebound$.

\subsection{Summary of our results}

In \cref{sec:reduction} we present a novel black-box reduction from SSP to finite-horizon MDPs (\cref{alg:ssp-reduction}), that yields $\sqrt{K}$ regret bounds when combined with a certain class of optimistic algorithms for regret minimization in finite-horizon MDPs that we call \emph{admissible} (\cref{def:admissible-alg}).
The regret analysis for the reduction is described in \cref{sec:regret-analysis}, and in \cref{sec:finite-horizon-analysis} we present an admissible algorithm for regret minimization in finite-horizon MDPs called \verb|ULCVI|.
We show that it guarantees the following optimal regret in the finite-horizon setting (stated formally in \cref{thm:ulcvi-guarantees}).
Note that (for large enough number of episodes) this bound depends only on the expected cost of the optimal policy and not on the horizon $H$.

\begin{theorem}         
    \label{thm:informal-ulcvi-guarantees}
    Running \verb|ULCVI| (\cref{alg: ulcvi RL} in \cref{sec:finite-horizon-analysis}) in a finite-horizon MDP guarantees, with probability at least $1 - \delta$, a regret bound of 
    \[
        O \brk3{ \sqrt{ (\costbound^2 + \costbound) |S||A| M} \log \frac{M H |S| |A|}{\delta} + H^4 \costbound^{-1} |S|^2 |A| \log^{3/2} \frac{M H |S| |A|}{\delta}},
    \] 
    for any number of episodes $M \ge 1$ simultaneously.
\end{theorem}

Combining \verb|ULCVI| with our reduction yields the following minimax optimal regret bound for SSP.

\begin{theorem}
    \label{thm:optimal-regret-bound} 
    Running the reduction in \cref{alg:ssp-reduction} with the finite-horizon regret minimization algorithm \verb|ULCVI| ensures, with probability at least $1 - \delta$, 
    \[
        \regret 
        = 
        O \brk3{ \sqrt{(\costbound^2 + \costbound) |S| |A| K} \log \frac{K \timebound |S| |A|}{\delta} + \timebound^5 \costbound^{-2} |S|^2 |A| \log^6 \frac{K \timebound |S| |A|}{\delta}}.
    \]
\end{theorem}

\begin{remark}
    An important observation is that this regret bound is meaningful even for small $K$.
    Unlike finite-horizon MDPs, where linear regret is trivial, in SSP ensuring finite regret is not easy.
    Our regret bound also implies that if we play for only one episode, i.e., we are only interested in the time it takes to reach the goal state, then it will take us at most $\wt O(\timebound^5 \costbound^{-2} |S|^2 |A|)$ time steps to do so.
\end{remark}

\begin{remark}
    Note that our algorithm needs to know an upper bound on $\timebound$ in advance.
    However, if all costs are strictly positive (i.e., at least $\cmin > 0$), then there is a trivial upper bound of $\costbound / \cmin$.
    In this case, our algorithm keeps an optimal regret bound for large enough $K$, since the bound on $\timebound$ only appears in the additive factor.
    Some previous work used a perturbation argument to generalize their results from the $\cmin$ case to general costs \citep{tarbouriech2019noregret,rosenberg2020near,rosenberg2020stochastic}.
    In our case, it will not work since the dependence on $1 / \cmin$ in the additive term is too large.
    This may be an inherent shortcoming of using finite-horizon reduction to solve SSPs, as it also appears in the works of \citet{chen2020minimax,chen2021finding} for the adversarial setting.
\end{remark}

\begin{remark}
    In practice, one can think of $\timebound$ as a parameter of the algorithm that controls computational complexity and the number of steps to complete $K$ episodes.
    By choosing the parameter $\timebound = x$ for example, we can guarantee that the regret bound of \cref{thm:optimal-regret-bound} holds against the best proper policy with expected time to the goal of at most $x$ (assuming there exists one), and we can also guarantee that the total computational complexity of the algorithm is $\wt O (x \log K)$ (see \cref{remark:complexity-reduction}).
    Furthermore, the algorithm will take at most $\wt O (x K + poly(x,|S|,|A|))$ steps to complete $K$ episodes.
\end{remark}

\begin{remark}
    While the additive term in our regret bound is standard for most cases, it becomes large  when $\costbound$ is extremely small because of the dependence in $\costbound^{-1}$.
    This was not an issue in previous work \citep{tarbouriech2019noregret,rosenberg2020near} since they assumed that the costs are deterministic and known.
    We believe that this dependence is an artifact of our analysis that may be avoided with a more careful definition of $\knownthresh$ (see \cref{def:admissible-alg}) that depends on the actual cost in each state-action pair and not just $\costbound$.
    Nevertheless, the main focus of this paper is on establishing that the minimax optimal regret for SSP is $\wt \Theta (\sqrt{(\costbound^2 + \costbound) |S| |A| K})$, and not on optimizing lower order terms.
    By that we also show that this is the minimax optimal regret for finite-horizon which is independent of the horizon $H$ (up to logarithmic factors).
    Tightening the additive term and eliminating its dependence in $\costbound^{-1}$ is left as an interesting future direction.
\end{remark}

In \cref{sec:lower-bound-appendix} we prove that our regret bound is indeed minimax optimal.
To complement the $\Omega(\costbound \sqrt{ |S| |A| K})$ lower bound of \citet{rosenberg2020near} that assumes $\costbound \ge 1$, we provide the following tighter lower bound for the case that $\costbound < 1$.

\begin{theorem}
    \label{thm:lowerbound}
    Let $\costbound \le \frac12$.
    There exists an SSP problem instance $\calM = (S, A, P, c, \sinit, \ssink)$ in which $\ctgopt(s) \le \costbound$ for all $s \in S$, $|S| \ge 2$, $|A| \ge 2$, $K \ge \costbound |S| |A|$, such the expected regret of any learner after $K$ episodes satisfies
    \[
        \bbE[\regret] \ge \frac{1}{32} \sqrt{\costbound |S| |A| K}.
    \]
\end{theorem}

\section{A black-box reduction from SSP to finite-horizon}
\label{sec:reduction}

Our algorithm takes as input an algorithm $\calA$ for regret minimization in finite-horizon MDPs, and uses it to perform a black-box reduction. The algorithm is depicted below as \cref{alg:ssp-reduction}.

The algorithm breaks the individual time steps that comprise each of the $K$ episodes into \emph{intervals} of $H$ time steps.
If the agent reaches the goal state before $H$ time steps, we simply assume that she stays in $\ssink$ until $H$ time steps are elapsed.
We see each interval as one episode of a finite-horizon model $\wh \calM = (\wh S,A,\wh P,H,\hat c, \hat c_f)$, where $\wh S = S \cup \{ \ssink \}$ and $\hat c_f: \wh S \rightarrow \bbR$ is a set of terminal costs defined by $\hat c_f(s) = 8 \costbound \indevent{s \ne \ssink}$, where $\indevent{s \ne \ssink}$ is the indicator function that equals 1 if $s \ne \ssink$ and 0 otherwise. Moreover, $\wh P,\hat c$ are the natural extensions of $P,c$ to the goal state. That is, $\hat c(s,a) = c(s,a) \indevent{s \ne \ssink}$ and
\[
    \wh P(s' \mid s,a)
    =
    \begin{cases}
        P(s' \mid s,a), \quad & s \ne \ssink;
        \\
        1, & s = \ssink, s'=\ssink;
        \\
        0, & s = \ssink, s' \ne \ssink.
    \end{cases}
\]
The horizon $H$ (which we will set to be roughly $\timebound$) is chosen such that the optimal SSP policy will reach the goal state in $H$ time steps with high probability (recall that the expected hitting time of the optimal policy is bounded by $\timebound$). 
The additional terminal cost is there to encourage the agent to reach the goal state within $H$ steps, which otherwise is not necessarily optimal with respect to the planning horizon.

\begin{algorithm}[H]
    \caption{\sc Reduction from SSP to finite-horizon MDP}
    \label{alg:ssp-reduction}
    \begin{algorithmic}[1] 
        
        \STATE {\bfseries input:} state sapce $S$, action space $A$, initial state $\sinit$, goal state $\ssink$, confidence parameter $\delta$, number of episodes $K$, bound on the expected cost of the optimal policy $\costbound$, bound on the expected time of the optimal policy $\timebound$ and algorithm $\calA$ for regret minimization in finite-horizon MDPs.
        
        \STATE {\bfseries initialize} $\calA$ with state space $\wh S = S \cup \{ \ssink \}$, action space $A$, horizon $H = 8 \timebound \log (8K)$, confidence parameter $\delta/4$, terminal costs $\hat c_f(s) = 8 \costbound \indevent{s \ne \ssink}$ and bound on the expected cost of the optimal policy $9 \costbound$.
        
        \STATE {\bfseries initialize} intervals counter $m \gets 0$ and time steps counter $t \gets 1$.
         
        \FOR{$k=1,\dots,K$}
            
            \STATE set $s_t \gets \sinit$.
            
            \WHILE{$s_t \neq \ssink$}
            
                \STATE set $m \gets m+1$, feed initial state $s_t$ to $\calA$ and obtain policy $\pi^m = \{ \pi^m_h: \wh S \to A \}_{h=1}^H$.
        
                \FOR{$h=1,\dots,H$}
        
                    \STATE play action $a_t = \pi^m_h(s_t)$, suffer cost $C_t \sim c(s_t,a_t)$, and set $s^m_h=s_t,a^m_h=a_t,C^m_h=C_t$.
            
                    \STATE observe next state $s_{t+1} \sim P(\cdot \mid s_t,a_t)$ and set $t \gets t+1$.
            
                    \IF{$s_t = \ssink$}
            
                        \STATE pad trajectory to be of length $H$ and BREAK.
            
                    \ENDIF
        
                \ENDFOR
                
                \STATE set $s^m_{H+1} = s_t$.
                
                \STATE feed trajectory $\traj{m} = (s^m_1,a^m_1,\dots,s^m_H,a^m_H,s^m_{H+1})$ and costs $\{ C^m_h \}_{h=1}^H$ to $\calA$.
        
            \ENDWHILE
        \ENDFOR
    \end{algorithmic}
\end{algorithm}

The algorithm $\calA$ is initialized with the state and action spaces as in the original SSP instance, the horizon length $H$, a confidence parameter $\delta / 4$, a set of terminal costs $\hat c_f$ and a bound on the expected cost of the optimal policy in the finite-horizon model $9 \costbound$.
At the beginning of each interval, it takes as input an initial state and outputs a policy to be used throughout the interval.
In the end of the interval it receives the trajectory and costs observed through the interval.

Note that while \cref{alg:ssp-reduction} may run any finite-horizon regret minimization algorithm, in the analysis we require that $\calA$ possesses some properties (that most optimistic algorithms already have) in order to establish our regret bound.
We specifically require $\calA$ to be an \emph{admissible} algorithm---a model-based optimistic algorithm for regret minimization in finite-horizon MDPs, e.g., \verb|UCBVI| \citep{azar2017minimax} and \verb|EULER| \citep{zanette2019tighter}. Admissible algorithms are defined formally as follows.

\begin{definition}
    \label{def:admissible-alg}
    A model-based algorithm $\calA$ for regret minimization in finite-horizon MDPs is called \emph{admissible} if, when running $\calA$ with confidence parameter $\delta$, there is a good event that holds with probability at least $1 - \delta$, under which the following hold: 
    \begin{enumerate}[label=(\roman*), nosep]
        \item $\calA$ provides anytime regret guarantees without prior knowledge of the number of episodes, and when the initial state of each episode is arbitrary.
        The regret bound that $\calA$ guarantees for $M$ episodes is denoted by $\wh \calR_{\calA} (M)$, for some non-decreasing function $\wh \calR_{\calA}$.
        
        \item The policy $\pi^m$ that $\calA$ picks in episode $m$ is greedy with respect to an estimate of the optimal policy's $Q$-function.
        
        \item The algorithm's estimate $\underline{J}^m$ of $\hatctg{\star}$ (the cost-to-go function associated with the optimal finite-horizon policy) is optimistic, i.e., $\underline{J}_h^m(s) \le \hatctg{\star}_h(s)$ for every $s \in S$ and $h=1,\dots,H+1$.
        
        \item $\calA$ computes $\underline{J}^m$ using estimates $\tilde c^m,\wt P^m$ of the cost function $\hat c$ and the transition function $\wh P$, respectively.
        There exists $\knownthresh$ which is a function of $H,|S|,|A|$ such that: if state-action pair $(s,a)$ was visited at least $\knownthresh \log \frac{M H |S| |A|}{\delta}$ times, then  $| \tilde c^m_h(s,a) - \hat c(s,a) | \le \costbound/H$ and $\lVert \wt P^m(\cdot \mid s,a) - \wh P(\cdot \mid s,a) \rVert_1 \le 1/(9H)$.
    \end{enumerate}
\end{definition}

%
Using an admissible algorithm in \cref{alg:ssp-reduction} enables us to bound the total number of intervals, thus ensuring that the agent reaches the goal state in almost every interval. This is because, as $\calA$ is optimistic, it will try to avoid the terminal cost (which is suffered in all states except for $\ssink$) by reaching the goal state. In addition, $\calA$ will succeed in doing so once it has a good enough estimation of the transition function.
%
Armed with the notion of admissibility, in the sequel we prove the following regret bound for any admissible algorithm $\calA$.
The proof of \cref{thm:optimal-regret-bound} is now given by combining \cref{thm:regret-bound-with-admissible-algorithm} with the regret bound of \verb|ULCVI| in \cref{thm:informal-ulcvi-guarantees}.

\begin{theorem}
    \label{thm:regret-bound-with-admissible-algorithm}
    Let $\calA$ be an admissible algorithm for regret minimization in finite-horizon MDPs and denote its regret in $M$ episodes by $\wh \calR_\calA (M)$.
    Then, running \cref{alg:ssp-reduction} with $\calA$ ensures that, with probability at least $1 - \delta$, 
    \begin{align*}
        \regret 
        & \le
        \wh \calR_\calA \brk*{4 K + 4\cdot 10^4 |S| |A| \knownthresh \log \frac{K \timebound |S| |A| \knownthresh}{\delta}}
        \\
        & \qquad + 
        O \brk*{ \sqrt{ (\costbound^2 + \costbound) K \log \frac{K \timebound |S| |A| \knownthresh}{\delta}} + \timebound \knownthresh |S| |A| \log^2 \frac{K \timebound |S| |A| \knownthresh}{\delta}},
    \end{align*}
    where $\knownthresh$ is a quantity that depends on the algorithm $\calA$ and on $|S|,|A|,H$.
\end{theorem}

\begin{remark}[Computational complexity]
    \label{remark:complexity-reduction}
    Our reduction directly inherits the computational complexity of the finite-horizon algorithm $\calA$ in $\numintervals$ episodes, where $\numintervals \approx K + poly(|S|,|A|,\timebound)$ by \cref{lem:bound-on-number-of-intervals}.
    The computational complexity of \verb|ULCVI| is $O(H |S|^3 |A|^2 \log (MH))$, and therefore our optimal regret for SSP is achieved in total computational complexity of $O \bigl( \timebound |S|^3 |A|^2 \log^2 \frac{K \timebound |S| |A|}{\delta} \bigr)$ which is only logarithmic in the number of episodes.
\end{remark}

\subsection{Unknown expected optimal cost \texorpdfstring{$\costbound$}{}}

Inspired by techniques for estimation of the SSP-diameter in the adversarial SSP literature \citep{rosenberg2020stochastic,chen2021finding}, in \cref{sec:unknown-B-appendix} we show that our reduction does not need to know $\costbound$ in advance, but can instead estimate it on the fly.

We can obtain a reasonable estimate (up to a constant multiplicative factor) of the cost-to-go from state $s$ by running the \verb|Bernstein-SSP| algorithm of \citet{rosenberg2020near} for regret minimization in SSPs (that does not need to know $\costbound$) with initial state $s$ for roughly $\timebound^2 |S|^2 |A|$ episodes.
Thus, we can apply our reduction while utilizing our first visits to each state in order to estimate its cost-to-go.

We operate in \emph{phases} where each phase ends when some state is visited at least $\timebound^2 |S|^2 |A|$ times, and all states that were not visited enough are treated as the goal state.
Once we reach a poorly visited state, we simply run an episode of the corresponding \verb|Bernstein-SSP| algorithm.
Notice that this comes at a computational cost that is independent of the number of episodes $K$ (since we use \verb|Bernstein-SSP| for a small number of episodes), and in \cref{sec:unknown-B-appendix} we show that it achieves similar regret bounds with only an additional additive factor of $\wt O(\timebound^3 |S|^3 |A|)$.

\section{Regret analysis}
\label{sec:regret-analysis}

In this section we prove \cref{thm:regret-bound-with-admissible-algorithm}.
Below we give a high-level overview of the proofs and defer the details to \cref{sec:reduction-proofs}.
We start the analysis with a regret decomposition that states that the SSP regret can be bounded by the sum of two terms: the expected regret of the finite-horizon algorithm, and the deviation of the actual cost in each interval from its expected value. 
To that end, we use the notations: $M$ for the total number of intervals, $\traj{m} = (s_1^m,a_1^m,\dots,s_h^m,a_h^m,s_{H+1}^m)$ for the trajectory visited in interval $m$, $C^m_h$ for the cost suffered in step $h$ of interval $m$, $\pi^m$ for the policy chosen by $\calA$ for interval $m$, and $\hatctg{\pi}_h(s)$ for the expected finite-horizon cost when playing policy $\pi$ starting from state $s$ in time step $h$.

\begin{restatable}{lemma}{regrettofinitehorizonregret}
    \label{lem:regret-to-finite-horizon-regret}
    For $H = 8 \timebound \log (8K)$, we have the following bound on the regret of \cref{alg:ssp-reduction}:
    \begin{align}
        \label{eq:regret-decomposition}
        \regret 
        \le 
        \wh \calR_\calA(\numintervals)
        + 
        \sum_{m=1}^\numintervals \brk*{\sum_{h=1}^H C^m_h + \hat c_f(s_{H+1}^m) - \hatctg{\pi^m}_1 (s^m_1)} + \costbound.
    \end{align}
\end{restatable}

The bound in \cref{eq:regret-decomposition} is comprised of two summands and an additional constant.
The first summand is an upper bound on the expected finite-horizon regret which we acquire by the admissibility of $\calA$ (\cref{def:admissible-alg}). Note that this bound is in terms of the number of intervals $\numintervals$ (i.e., the number of finite-horizon episodes) which is a random variable and not necessarily bounded.
In what follows we show that, using the admissibility of $\calA$, we can actually bound $M$ by the number of SSP episodes $K$ plus a constant that depends on $\knownthresh, |S|,|A|,\timebound$ (but not on $K$). 
The second summand in \cref{eq:regret-decomposition} relates to the deviation of the total finite-horizon cost from its expected value.

The proof of \cref{lem:regret-to-finite-horizon-regret} builds on two key ideas.
The first is that, by setting $H$ to be $O(\timebound \log K)$, we ensure that the expected cost of the optimal policy in the SSP model $\calM$ is close to that in the finite-horizon model $\wh \calM$.
The second idea is that if the agent does not reach the goal state in a certain interval, then she must suffer the terminal cost in the finite-horizon model. Therefore, although in a single episode there may be many intervals in which the agent does not reach the goal state, we can upper bound the cost in these extra intervals in $\calM$ by the corresponding terminal costs in $\wh \calM$.


Next, we bound the deviation of the actual cost in each interval from its expected value which appears as the second summand in \cref{eq:regret-decomposition}. The bound is due to the following lemma.

\begin{lemma}
    \label{lem:cost-deviation-from-value-function}
    Assume that the reduction is performed using an admissible algorithm $\calA$.
    Then, the following holds with probability at least $1 - \nicefrac{3 \delta}{8}$,
    \[
        \sum_{m=1}^\numintervals \brk*{\sum_{h=1}^H C^m_h + \hat c_f(s_{H+1}^m) - \hatctg{\pi^m}_1(s_1^m)}
        = 
        O \brk*{ \sqrt{ (\costbound^2 + \costbound) \numintervals \log \frac{M}{\delta}} + H \knownthresh |S| |A| \log \frac{M K \timebound |S| |A|}{\delta}}.
    \]
\end{lemma}

The key observation here relies on the notion of \emph{unknown} state-action pairs -- pairs that were not visited at least $\knownthresh$ times.
After $\knownthresh$ visits to some state-action pair $s,a$, we have a reasonable estimate of the next-state distribution $P(\cdot \mid s,a)$ therefore we can show that the expected accumulated cost in an interval until reaching an unknown state-action pair or the goal state is of order $\costbound$.
Moreover, the second moment of this cost is of order $\costbound^2 + \costbound$.
Thus, using Freedman inequality, we bound the deviation by $\wt O(\sqrt{(\costbound^2 + \costbound) M})$, plus a cost of $O(H)$ for each ``bad'' interval in which we do not reach an unknown state-action pair or the goal state (there are roughly $\knownthresh |S| |A|$ such intervals).

Lastly, we need to bound the number of intervals $\numintervals$ to obtain a regret bound in terms of $K$ and not $\numintervals$ (notice that $\numintervals$ is a random variable that is not bounded a-priori).

\begin{lemma}
    \label{lem:bound-on-number-of-intervals}
    Assume that the reduction is performed using an admissible algorithm $\calA$.
    Then, with probability at least $1 - \nicefrac{3 \delta}{8}$,
    $
        M 
        \le 
        4 K + 4\cdot 10^4 |S| |A| \knownthresh \log (K \timebound |S| |A| \knownthresh/\delta).
    $
\end{lemma}

The proof shows that in every interval there is a constant probability to reach either the goal state or an unknown state-action pair.
Leveraging this observation with a concentration inequality, we can bound the number of intervals by $\wt O (K + \knownthresh |S| |A| H)$.

We can now prove a bound on the regret of \cref{alg:ssp-reduction} using any admissible algorithm $\calA$.

\begin{proof}[Proof of \cref{thm:regret-bound-with-admissible-algorithm}]
    The regret bound of $\calA$, \cref{lem:bound-on-number-of-intervals,lem:cost-deviation-from-value-function} all hold with probability at least $1 - \delta$, via a union bound.
    Using \cref{lem:regret-to-finite-horizon-regret,lem:cost-deviation-from-value-function} we can write
    \begin{align*}
        \regret 
        & \le
        \wh \calR_{\calA}(M) 
        + 
        O \brk4{ \sqrt{(\costbound^2 + \costbound)\numintervals \log \frac{M}{\delta}} + H \knownthresh |S| |A| \log \frac{M K \timebound |S| |A|}{\delta}}
        +
        \costbound.
    \end{align*}
    Finally, we use \cref{lem:bound-on-number-of-intervals} to bound $\numintervals$ by $4 K + 4\cdot 10^4 |S| |A| \knownthresh \log (K \timebound |S| |A| \knownthresh/\delta)$.
\end{proof}

\section{ULCVI: an admissible algorithm for finite-horizon MDPs}
\label{sec:finite-horizon-analysis}

In this section we present the Upper Lower Confidence Value Iteration algorithm (\verb|ULCVI|; \cref{alg: ulcvi RL}) for regret minimization in finite-horizon MDPs. This result holds independently of our SSP algorithm.
Since the algorithm is similar to previous optimistic algorithms for the finite-horizon setting, e.g., \verb|UCBVI| \citep{azar2017minimax} and \verb|ORLC| \citep{dann2019policy}, we defer the analysis to \cref{sec:ulcvi-proofs} and focus on our technical novelty -- bounding the regret in terms of the optimal value function and not the horizon.

\begin{algorithm}[t]
    \caption{\textsc{Upper Lower Confidence Value Iteration (ULCVI)}} 
    \label{alg: ulcvi RL}
    \begin{algorithmic}[1]
        \STATE {\bf input:} state space $S$, action space $A$, horizon $H$, confidence parameter $\delta$, terminal costs $\hat c_f$ and upper bound on the expected cost of the optimal policy $\costbound$.

        \STATE {\bf initialize:} $n^0(s,a)=0,n^0(s,a,s')=0,N^0(s,a)=0,N^0(s,a,s')=0 \  \forall (s,a,s') \in S \times A \times S$.
        
        \STATE {\bf initialize:} $C^0(s,a)=0,\bar c^0(s,a) = 0,\bar P^0(s' | s,a) = \indevent{s' = s} \  \forall (s,a,s') \in S \times A \times S$.
        
        \STATE {\bf initialize:} $\texttt{PlanningTrigger} = \texttt{true}$.

        \FOR{$m=1,2,\dots$}
            
            \STATE observe initial state $s^m_1$.
            
            \IF{$\texttt{PlanningTrigger} = \texttt{true}$}
            
                \STATE set $n^{m-1}(s,a) \gets N^{m-1}(s,a), n^{m-1}(s,a,s') \gets N^{m-1}(s,a,s') \ \forall (s,a,s').
                $
                
                \STATE set $\bar P^{m-1}(s' | s,a) \gets \frac{n^{m-1}(s,a,s')}{\max \{ 1 ,  n^{m-1}(s,a) \}} , \bar c^{m-1}(s,a) \gets \frac{C^{m-1}(s,a)}{\max \{ 1 , n^{m-1}(s,a) \}} \ \forall (s,a,s')$.
            
                 \STATE compute $\{ \pi^m_h(s) \}_{s,h}$ via \textsc{Optimistic-Pessimistic Value Iteration} (\cref{alg: optimistic pessimistic value iteration}).
                 
                 \STATE set $\texttt{PlanningTrigger} \gets \texttt{false}$.
                
            \ELSE
            
                \STATE set $n^{m-1}(s,a) \gets n^{m-2}(s,a),n^{m-1}(s,a,s') \gets n^{m-2}(s,a,s')  \ \forall (s,a,s') $
                
                \STATE set $\bar P^{m-1}(s' | s,a) \gets \bar P^{m-2}(s' | s,a) , \bar c^{m-1}(s,a) \gets \bar c^{m-2}(s,a) \ \forall (s,a,s')$.
            
                \STATE set $\pi^m_h(s) \gets \pi^{m-1}_h(s)$ for all $s \in S$ and $h=1,\dots,H$.
            
            \ENDIF
    
            \STATE set $N^m(s,a) \gets N^{m-1}(s,a),N^m(s,a,s') \gets N^{m-1}(s,a,s'),C^m(s,a) \gets C^{m-1}(s,a) \ \forall (s,a,s')$.
    
            \FOR{$h=1,\dots,H$}
    
                \STATE pick action $a^m_h = \pi^m_h(s^m_h)$.
                
                \STATE suffer cost $C^m_h \sim \hat c(s^m_h,a^m_h)$ and observe next state $s^m_{h+1} \sim \wh P(\cdot \mid s^m_h,a^m_h)$.
                
                \STATE update visits counters $n^m(s^m_h,a^m_h) \gets n^m(s^m_h,a^m_h) + 1 , n^m(s^m_h,a^m_h,s^m_{h+1})\gets n^m(s^m_h,a^m_h,s^m_{h+1}) + 1$.
                
                \STATE update accumulated cost $C^m(s^m_h,a^m_h) \gets C^m(s^m_h,a^m_h) + C^m_h$.
                
                \IF{$N^m(s^m_h,a^m_h) \ge 2 n^{m-1}(s^m_h,a^m_h)$}
                
                    \STATE set $\texttt{PlanningTrigger} \gets \texttt{true}$.
                
                \ENDIF
    
            \ENDFOR
            
            \STATE Suffer terminal cost $\hat c_f(s^m_{H+1})$.
        \ENDFOR
    \end{algorithmic}
\end{algorithm}

In each episode $m$, the \verb|ULCVI| algorithm maintains an optimistic lower bound $\underline{J}^m_h(s)$ and a pessimistic upper bound $\bar J^m_h(s)$ on the cost-to-go function of the optimal policy $J^\star_h(s)$, and acts greedily with respect to the optimistic estimates.
These optimistic and pessimistic estimates are computed based on the empirical transition function $\bar P^{m-1} (s' \mid s,a)$ and the empirical cost function $\bar c^{m-1}(s,a)$ to which we add an exploration bonus $b^m_c(s,a) + b^m_p(s,a)$, where $b^m_p$ handles the approximation error in the transitions and $b^m_c$ handles the approximation error in the costs.
The bonuses are defined as follows,
\begin{alignat}{2}
    \label{eq:bonus-definitions}
    &b^m_c(s,a)
    &&=
    \sqrt{ \frac{2 \overline{\VAR}^{m-1}_{s,a}(C) L_m }{\max \{ 1 , n^{m-1}(s,a) \}}} +  \frac{5 L_m}{\max \{ 1 , n^{m-1}(s,a) \}}
    \\
    \nonumber
    &b^m_p(s,a) 
    &&= \sqrt{\frac{2\VAR_{\bar{P}^{m-1}(\cdot \mid s,a)}(\underline{J}^m_{h+1}) L_m}{\max \{ 1 , n^{m-1}(s,a) \}}} + \frac{62 H^3 \costbound^{-1} |S| L_m}{\max \{ 1 , n^{m-1}(s,a) \}} + \frac{\costbound}{16 H^2}\E_{\bar{P}^{m-1}(\cdot \mid s,a)}\brs{\bar{J}^m_{h+1}(s') - \underline{J}^m_{h+1}(s')  },
\end{alignat}
where $L_m = 3 \log(3 |S| |A| H m/\delta)$ is a logarithmic factor and $n^{m-1}(s,a)$ is the number of visits to $(s,a)$ in the first $m-1$ episodes. 
Furthermore, $\overline{\VAR}^{m-1}_{s,a}(C)$ is the empirical variance of the observed costs in $(s,a)$ in the first $m-1$ episodes.\footnote{The empirical variance of $n$ numbers $a_1,\dots,a_n$ is defined by $\frac{1}{n} \sum_{i=1}^n \bigl( a_i - \frac{1}{n} \sum_{j=1}^n a_j \bigr)^2$.} 
Lastly, the term $\VAR_{\bar{P}^{m-1}(\cdot \mid  s,a)}(\underline{J}^m_{h+1})$ is the variance of the next state value $\underline{J}^m_{h+1}$ from state-action pair $(s,a)$, calculated via the empirical transition model, i.e.,
$    \VAR_{\bar{P}^{m-1}(\cdot \mid  s,a)}(\underline{J}^m_{h+1}) = \E_{\bar{P}^{m-1}(\cdot \mid s,a)}\brs{\underline{J}^m_{h+1}(s')^2} - \E_{\bar{P}^{m-1}(\cdot \mid s,a)}\brs{\underline{J}^m_{h+1}(s')}^2$.

For improved computational complexity, we compute the optimistic policy only in episodes in which the number of visits to some state-action pair was doubled.
This ensures that the number of optimistic policy computations grows only logarithmically with the number of episodes, i.e., it is bounded by $3 |S| |A| \log (MH)$.
Since each optimal policy computation costs $O(H |S|^2 |A|)$ in the finite-horizon MDP model, our algorithm enjoys a total computational complexity of $O (H |S|^3 |A|^2 \log (MH))$.

\begin{algorithm}[t]
    \caption{\textsc{Optimistic-Pessimistic Value Iteration}} 
    \label{alg: optimistic pessimistic value iteration}
    \begin{algorithmic}[1]
        \STATE {\bf input:} $n^{m-1}, \bar{P}^{m-1}, \bar c^{m-1},\hat c_f,\costbound$.

        \STATE {\bf initialize} $\underline{J}^m_{H+1}(s)=\bar{J}^m_{H+1}(s)=\hat c_f(s)$ for all $s\in S$.

        \FOR{$h=H,H-1,\ldots,1$}
            \FOR{$s\in S$}
                \FOR{$a \in A$}
                    \STATE set the bonus $b^m_h(s,a) = b^m_c(s,a) + b^m_p(s,a)$ defined in \cref{eq:bonus-definitions}.
                    
                    \STATE compute optimistic and pessimistic Q-functions:
                    \begin{alignat*}{2}
                        &\underline{Q}^m_h(s,a)
                        &&=
                        \bar c^{m-1}(s,a) - b^m_h(s,a) + \E_{\bar{P}^{m-1}(\cdot \mid s,a)}[\underline{J}^m_{h+1}(s')]
                        \\ 
                        &\bar{Q}^m_h(s,a) 
                        &&=
                        \bar c^{m-1}(s,a) + b^m_h(s,a) +\E_{\bar{P}^{m-1}(\cdot \mid s,a)}[\bar{J}^m_{h+1}(s')].
                    \end{alignat*}
                \ENDFOR
    
                \STATE $\pi^m_h(s)\in \arg\min_{a \in A} \underline{Q}^m_h(s,a)$.
        
                \STATE $\underline{J}^m_h(s) = \max\brc*{ \underline{Q}^m_h(s,\pi^m_h(s)),0}$,  $\bar{J}^m_h(s) = \min\brc*{ \bar{Q}^m_h(s,\pi^m_h(s)),H}$.
            \ENDFOR
        \ENDFOR
    \end{algorithmic}
\end{algorithm}

For clarity, we keep the notation of the finite-horizon MDP as $\wh \calM = (S,A,\wh P,H,\hat c,\hat c_f)$, and let $\costbound = \max_{s,h} \wh J^\star_h(s)$ where $\wh J^\pi$ is the value function of policy $\pi$ (in the case of our SSP reduction this parameter is simply $9 \costbound$ by \cref{lem:ctg-to-finite-horizon-value}).
This implies that $\hat c_f(s) \le \costbound$ for every $s$, and for simplicity, we assume that $\costbound \le H$.
Thus, the maximal total cost in an episode is bounded by $H + \costbound \le 2H$.
In \cref{sec:ulcvi-proofs} we prove the following high probability regret bound.

\begin{theorem}     \label{thm:ulcvi-guarantees}
    \verb|ULCVI| (\cref{alg: ulcvi RL}) is admissible with the following guarantees:
    \begin{enumerate}[label=(\roman*), nosep]
        \item With probability at least $1 - \delta$, the regret bound of \verb|ULCVI| is 
        \[
            \wh \calR_{\verb|ULCVI|} (M) 
            =
            O \brk*{ \sqrt{ (\costbound^2 + \costbound) |S||A| M} \log \frac{M H |S| |A|}{\delta} + H^4 \costbound^{-1} |S|^2 |A| \log^{3/2} \frac{M H |S| |A|}{\delta}}
        \]
        for any number of episodes $M \ge 1$.
        \item $\omega_{\verb|ULCVI|} = O(H^4 \costbound^{-2} |S|)$.
    \end{enumerate}
\end{theorem}

Our analysis resembles the one in~\citet{efroni2021confidence}, and is adapted to the stationary MDP setting (i.e., the transition function does not depend on the time step $h$), and to the setting where we have costs instead of rewards, and terminal costs (which do not appear in previous work).
By the definition of the algorithm and the regret bound in \cref{thm:ulcvi-guarantees}, it is clear that properties (i)-(iii) in \cref{def:admissible-alg} of admissible algorithms hold.
For property (iv), we use standard concentration inequalities and the definition of the bonuses in \cref{eq:bonus-definitions} in order to show it holds for $\omega_{\verb|ULCVI|} = O ( H^4 \costbound^{-2} |S| )$.

To obtain a regret bound whose leading term depends on $\costbound$ and not $H$, we start with a standard regret analysis for optimistic algorithms that establishes the regret scales with the square-root of the variance of the value functions of the agent's policies, i.e.,
\[
    \wh \calR_{\verb|ULCVI|} (M)
    \lesssim
    \sqrt{|S| |A|} \sqrt{ \sum_{m=1}^M \sum_{h=1}^H \VAR_{ P(\cdot | s^m_h,a^m_h)} ( J^{\pi^m}_{h+1})} + H^4 \costbound^{-1} |S|^2 |A|,
\]
up to logarithmic factors and lower order terms.
This can be further bounded by the second moment of the cumulative cost in each episode as follows,
\[
    \wh \calR_{\verb|ULCVI|} (M)
    \lesssim
    \sqrt{|S| |A|} \sqrt{ \sum_{m=1}^M \bbE \brk[s]*{\brk*{\sum_{h=1}^H C^m_h + \hat c_f(s^m_{H+1})}^2 ~\Bigg|~ \trajconcat{m}}} + H^4 \costbound^{-1} |S|^2 |A|,
\]
where $\trajconcat{m}$ is the sequence of state-action pairs observed up to episode $m$.
Leveraging our techniques for the SSP reduction (but independently), we show that the second moment of the cumulative cost until an unknown state-action pair is reached can be bounded by $O ( \costbound^2 + \costbound )$.
Therefore, we have at most $\wt O(H^4 \costbound^{-2} |S|^2 |A|)$ episodes in which we bound the second moment trivially by $O(H^2)$, and in the rest of the episodes we can bound it by $O ( \costbound^2 + \costbound )$.
Together this yields the theorem as follows,
\[
    \wh \calR_{\verb|ULCVI|} (M)
    \lesssim
    \sqrt{|S| |A|} \sqrt{ (\costbound^2 + \costbound) M + H^2 \cdot H^4 \costbound^{-2} |S|^2 |A|}
    \lesssim
    \sqrt{ (\costbound^2 + \costbound) |S| |A| M} + H^4 \costbound^{-1} |S|^2 |A|.
\]

\section*{Acknowledgements}

This project has received funding from the European Research Council (ERC) under the European Union’sHorizon 2020 research and innovation program (grant agreement No. 882396), by the Israel Science Foundation(grant number 993/17), Tel Aviv University Center for AI and Data Science (TAD), and the Yandex Initiative for Machine Learning at Tel Aviv University

\bibliographystyle{plainnat}
\bibliography{bib}

\clearpage
\appendix

\section{Proofs for Section~\ref{sec:regret-analysis}}
\label{sec:reduction-proofs}

\subsection{Proof of Lemma~\ref{lem:regret-to-finite-horizon-regret}}
\label{sec:relating-SSP-regret-to-finite-horizon-regret}

In this section we relate the SSP regret and the finite-horizon regret, which relies on \cref{lem:ctg-to-finite-horizon-value,lem:opt-ctg-to-finite-horizon-value} below that compare the cost-to-go function in the SSP $\calM$ to the value function in the finite-horizon $\wh \calM$. 
To that end, we define a cost-to-go function with respect to the finite-horizon MDP $\wh \calM$ as: $\hatctg{\pi}_h(s) = \bbE \brk[s]1{\sum_{h'=h}^H c(s_{h'}, a_{h'}) \mid s_h = s}$, for any deterministic finite-horizon policy $\pi : S \times [H] \mapsto A$.

\begin{lemma}
    \label{lem:ctg-to-finite-horizon-value}
    Let $\pi$ be a stationary policy.
    For every $s \in \wh S$ and $h = 1,\dots,H+1$ it holds that $$\hatctg{\pi}_h(s) \le \ctg{\pi}(s) + 8 \costbound \Pr [s_{H+1} \ne \ssink \mid s_h=s,\wh P,\pi].$$
\end{lemma}

\begin{proof}
    \begin{align*}
        \hatctg{\pi}_h(s)
        & =
        \sum_{h'=h}^H \sum_{s' \in \wh S} \Pr [s_{h'} = s' \mid s_h=s,\wh P,\pi] \; \hat c \brk1{s',\pi(s')} 
        + 
        \sum_{s' \in \wh S} \Pr [s_{H+1} = s' \mid s_h=s,\wh P,\pi] \; \hat c_f(s')
        \\
        & =
        \sum_{h'=h}^H \sum_{s' \in S} \Pr [s_{h'} = s' \mid s_h=s,P,\pi] \; c \brk1{s',\pi(s')}
        + 
        8 \costbound \; \Pr [s_{H+1} \ne \ssink \mid s_h=s,\wh P,\pi]
        \\
        & \le
        \sum_{h'=h}^\infty \sum_{s' \in S} \Pr [s_{h'} = s' \mid s_h=s,P,\pi] \; c \brk1{s',\pi(s')} 
        + 
        8 \costbound \; \Pr [s_{H+1} \ne \ssink \mid s_h=s,\wh P,\pi]
        \\
        & =
        \ctg{\pi}(s) + 8 \costbound \; \Pr [s_{H+1} \ne \ssink \mid s_h=s,\wh P,\pi]. \qedhere
    \end{align*}
\end{proof}

\begin{lemma}         
    \label{lem:opt-ctg-to-finite-horizon-value}
    For every $s \in \wh S$, it holds that $\ctgopt(s) \ge \hatctg{\piopt}_1(s) - \frac{\costbound}{K}$.
\end{lemma}

\begin{proof}
    The probability that $\piopt$ does not reach the goal in $H$ steps is at most $1/(8 K)$ due to \citet[Lemma 7]{chen2020minimax}. 
    Plugging that into \cref{lem:ctg-to-finite-horizon-value} yields the desired result.
\end{proof}

\begin{proof}[Proof of \cref{lem:regret-to-finite-horizon-regret}]
    Consider the first interval of the first episode.
    If it ends in the goal state then
    \[
        \sum_{i=1}^{\Ikk{1}} C^1_i
        =
        \sum_{h=1}^H C^1_h + \hat c_f(\ssink)
        =
        \sum_{h=1}^H C^1_h + \hat c_f(s_{H+1}^1).
    \]
    If the agent did not reach $\ssink$ in the first interval, then the agent also suffered the $8 \costbound$ terminal cost and thus
    \begin{align*}
        \sum_{i=1}^{\Ikk{1}} C^1_i
        & =
        \sum_{h=1}^H C^1_h + \hat c_f(s_{H+1}^1) + \sum_{i=H+1}^{\Ikk{1}} C^1_i - \hat c_f(s_{H+1}^1)
        \\
        & =
        \sum_{h=1}^H C^1_h + \hat c_f(s_{H+1}^1) + \sum_{i=H+1}^{\Ikk{1}} C^1_i - 8 \costbound
        \\
        & \le
        \sum_{h=1}^H C^1_h + \hat c_f(s_{H+1}^1) + \sum_{i=H+1}^{\Ikk{1}} C^1_i - \hatctg{\piopt}_1(s^1_{H+1}),
    \end{align*}
    where the last inequality follows by combining \cref{lem:opt-ctg-to-finite-horizon-value} with our assumption that $\ctgopt(s) \le \costbound$. 
    
    Repeating this argument iteratively we get, for every episode $k$,
    \begin{align*}
        \sum_{i=1}^{\Ik} C^k_i - \ctgopt(\sinit)
        & \le
        \sum_{i=1}^{\Ik} C^k_i - \hatctg{\piopt}_1(s_1^m) + \frac{\costbound}{K}
        \\
        & \le
        \sum_{m \in \numintervals_k} \sum_{h=1}^H C^m_h + \hat c_f(s_{H+1}^m) - \hatctg{\piopt}_1(s_1^m) + \frac{\costbound}{K} \\
        &=
        \sum_{m \in \numintervals_k} \brk4{\sum_{h=1}^H C^m_h + \hat c_f(s_{H+1}^m) - \hatctg{\pi^m}(s_1^m)}
        + 
        \sum_{m \in \numintervals_k} \brk4{\hatctg{\pi^m}(s_1^m) - \hatctg{\piopt}_1(s_1^m)}
        + 
        \frac{\costbound}{K},
    \end{align*}
    where $\numintervals_k$ is the set of intervals that are contained in episode $k$, and the first inequality follows from \cref{lem:opt-ctg-to-finite-horizon-value}.
    Summing over all episodes obtains
    \[
        \regret 
        \le
        \sum_{m=1}^M \brk4{\sum_{h=1}^H C^m_h + \hat c_f(s_{H+1}^m) - \hatctg{\pi^m}(s_1^m)}
        + 
        \sum_{m=1}^M \brk4{\hatctg{\pi^m}(s_1^m) - \hatctg{\piopt}_1(s_1^m)}
        + 
        \frac{\costbound}{K}.
    \]
    Notice that the second summand in the bound above is exactly the expected finite-horizon regret over the $M$ intervals.
    We finish the proof of the lemma by using the regret guarantees of $\calA$ (\cref{def:admissible-alg}).
\end{proof}

\subsection{Proof of Lemma~\ref{lem:cost-deviation-from-value-function}}
\label{sec:cost-deviation}

In this section we bound the deviation of the actual cost in each interval from its expected value. 
To do that, we apply \cref{lem:reach-goal-or-unknown-wp-1/2} below to bound the second moment of the cumulative cost in an interval up until an unknown state-action pair or the goal state were reached. 
Here $\trajconcat{m}$ denotes the union of all information prior to the $m^{th}$ interval together with the first state of the $m^{th}$ interval (more formally, $\{\trajconcat{m}\}_{m\geq 1}$ is a filtration).
Moreover, we denote by $h_m$ the last time step before an unknown state-action pair or the goal state were reached in interval $m$ (or $H$ if they were not reached).

\begin{lemma}
    \label{lem:reach-goal-or-unknown-wp-1/2}
    Let $m$ be an interval and assume that the reduction is performed using an admissible algorithm $\calA$.
    If the good event of $\calA$ holds until the beginning of interval $m$, then the agent reaches the goal state or an unknown state-action pair with probability at least $\frac12$.
    Moreover, denote by $C^m = \sum_{h=1}^{h_m} C^m_h + \hat c_f(s^m_{H+1}) \indevent{h_m = H}$ the cumulative cost in the interval until time $h_m$. 
    Then,
    $\bbE\brk[s]{(C^m)^2 \mid \trajconcat{m}} \le 2 \cdot 10^5 \costbound^2 + 4 \costbound$.
\end{lemma}

\begin{proof}
    The result is given by bounding the total expected cost suffered by the agent in another MDP (defined below) where all unknown state-action pairs are contracted with the goal state. The cost in this MDP is exactly $C^m$ by definition. 
    
    Let $\pi^m$ be the optimistic policy chosen by the algorithm for interval $m$.
    Consider the following finite-horizon MDP $\wh \calM^m = (\wh S,A,\wh P^m,H,\hat c,\hat c_f)$ that contracts unknown state-action pairs with the goal:
    \[
        \wh P^m_h(s' \mid s,a)
        =
        \begin{cases}
            0, \quad & (s',\pi^m_{h+1}(s')) \text{ is unknown};
            \\
            P(s' \mid s,a), \quad & s' \ne \ssink \text{ and } (s',\pi^m_{h+1}(s')) \text{ is known};
            \\
            1 - \sum_{s'' \in \wh S \setminus \{\ssink\}} \wh P^m_h(s'' \mid s,a), & s' = \ssink.
        \end{cases}
    \]
    
    Denote by $J^m$ the cost-to-go function of $\pi^m$ in the finite-horizon MDP $\wh \calM^m$.
    Further, let $\wt P'^m$ be the transition function induced by $\wt P^m$ in the MDP $\wh \calM^m$ similarly to $\wh P^m$, and $\wt J^m$ the cost-to-go function of $\pi^m$ with respect to $\wt P'^m$ (and with cost function $\tilde c^m$). 
    Notice that $\pi^m$ can only reach the goal state quicker in $\wh \calM^m$ than in $\wh \calM$, so that $\wt J^m_h(s) \le \underbar J_h^m(s) \le \hatctg{\piopt}_h(s)$ for any $s \in \wh S$.
    By the value difference lemma (see, e.g., \citealp{efroni2020optimistic}), for every $s,h$ such that $(s,\pi^m_h(s))$ is known,
    \begin{align*}
        & J^m_h(s)
        =
        \wt J^m_h(s) + \sum_{h'=h}^H \bbE \Bigl[ \hat c(s_{h'},a_{h'}) - \tilde c^m_{h'}(s_{h'},a_{h'}) \mid s_h=s,\wh P^m,\pi^m \Bigr] 
        \\
        & \qquad + 
        \sum_{h'=h}^H \bbE \Bigl[ \bigl( \wh P^m_{h'}(\cdot \mid s_{h'},a_{h'}) - \wt P'^m_{h'}(\cdot \mid s_{h'},a_{h'}) \bigr) \cdot \wt J^m \mid s_h=s,\wh P^m,\pi^m \Bigr]
        \\
        & \le
        \wt J^m_h(s) + H \max_{\substack{(s,\pi^m_{h'}(s)) \\ \text{known}}} | c(s,\pi^m_{h'}(s)) - \tilde c^m_{h'}(s,\pi^m_{h'}(s)) | 
        + 
        H \lVert \wt J^m \rVert_\infty \max_{\substack{(s,\pi^m_{h'}(s)) \\ \text{known}}} \lVert \wh P^m_{h'}(\cdot | s,\pi^m_{h'}(s)) - \wt P'^m_{h'}(\cdot | s,\pi^m_{h'}(s))  \rVert_1
        \\
        & \overset{(a)}{\le}
        \hatctg{\piopt}_h(s) + H \max_{\substack{(s,\pi^m_{h'}(s)) \\ \text{known}}} | c(s,\pi^m_{h'}(s)) - \tilde c^m_{h'}(s,\pi^m_{h'}(s)) | 
        \\
        & \qquad + 
        H \lVert \hatctg{\piopt}_h(s) \rVert_\infty \max_{\substack{(s,\pi^m_{h'}(s)) \\ \text{known}}} \lVert \wh P(\cdot | s,\pi^m_{h'}(s)) - \wt P^m(\cdot | s,\pi^m_{h'}(s)) \rVert_1
        \\
        & \le
        \hatctg{\piopt}_h(s) + H \max_{\substack{(s,\pi^m_{h'}(s)) \\ \text{known}}} | c(s,\pi^m_{h'}(s)) - \tilde c^m_{h'}(s,\pi^m_{h'}(s)) | 
        + 
        9 H \costbound \max_{\substack{(s,\pi^m_{h'}(s)) \\ \text{known}}} \lVert \wh P(\cdot | s,\pi^m_{h'}(s)) - \wt P^m(\cdot | s,\pi^m_{h'}(s)) \rVert_1,
    \end{align*}
    where the last inequality follows by optimism and since $\hatctg{\piopt}_h(s) \le 9 \costbound$ (\cref{lem:ctg-to-finite-horizon-value}), and (a) follows because
    \begin{align*}
        \lVert & \wh P^m_h(\cdot | s,a) - \wt P'^m_h(\cdot | s,a)  \rVert_1
        =        
        \sum_{\substack{(s',\pi^m_{h+1}(s')) \\ \text{known}}} | \wh P^m_h(s' | s,a) - \wt P'^m_h(s' | s,a) | + | \wh P^m_h(\ssink | s,a) - \wt P'^m_h(\ssink | s,a) |
        \\
        & =
        \sum_{\substack{(s',\pi^m_{h+1}(s')) \\ \text{known}}} | \wh P(s' | s,a) - \wt P^m(s' | s,a) | 
        + 
        \Bigl| \sum_{\substack{(s',\pi^m_{h+1}(s')) \\ \text{unknown}}} \wh P(s' | s,a) + \wh P(\ssink | s,a) - \wt P^m(s' | s,a) - \wt P^m(\ssink | s,a) \Bigr|
        \\
        & \le
        \lVert \wh P(\cdot | s,a) - \wt P^m(\cdot | s,a) \rVert_1.
    \end{align*} 
    Thus $J^m_h(s) \le \hatctg{\piopt}_h(s) + 2 \costbound$ since the number of visits to each known state-action pair is at least $\knownthresh \log \frac{M H |S| |A|}{\delta}$ and by property (iv) of admissible algorithms (\cref{def:admissible-alg}). 
    Also note that $J^m_h(s) \le 11 \costbound$ by \cref{lem:ctg-to-finite-horizon-value}, and for $h=1$ in particular we use \cref{lem:opt-ctg-to-finite-horizon-value} to obtain $J^m_1(s) \le 4 \costbound$.
    
    By Markov inequality, the probability that the agent suffers a cost of more than $8 \costbound$ in $\wh \calM^m$ is at most~$\frac12$.
    Notice that all costs are non-negative and there is a terminal cost of $8 \costbound$ in all states but the goal, therefore the agent cannot suffer a cost of less than $8 \costbound$ unless she reaches the goal.
    So the probability to reach the goal is at least $\frac12$.
    Moreover, note that the probability to reach the goal in $\wh \calM^m$ is equal to the probability to reach the goal or an unknown state-action pair in $\wh \calM$.
    
    
    Similarly, we notice that $\bbE \brk[s]{(C^m)^2 \mid \trajconcat{m}} = \bbE \brk[s]{(\wh C)^2}$, where $\wh C$ is the cumulative cost in $\wh \calM^m$, and we override notation by denoting $\wh C = \sum_{h=1}^H C_h + \hat c_f(s_{H+1})$. We have that,
    \begin{align*}
        \bbE \brk[s]{(\wh C)^2}
        &=
        \bbE \brk[s]4{\brk4{\sum_{h=1}^H C_h + \hat c_f(s_{H+1})}^2}  \\
        &=
        \bbE \brk[s]4{\brk4{\sum_{h=1}^{H-1} C_h + \hat c(s_H,a_H) + \hat c_f(s_{H+1})}^2} 
        \\
        & \quad +
        2 \bbE \brk[s]4{\brk4{\sum_{h=1}^{H-1} C_h + \hat c(s_H,a_H) + \hat c_f(s_{H+1})}\brk{C_H - \hat c(s_H,a_H)}}+
        \bbE \brk[s]{\brk{C_H - \hat c(s_H,a_H)}^2}.
    \end{align*}
    The second summand is zero since the realization of $C_H$ is independent of all other randomness given $s_H$.
    Also, since $C_H \in [0,1]$, the third summand satisfies 
    \[
        \bbE \brk[s]{\brk{C_H - \hat c(s_H,a_H)}^2}
        \le
        \bbE \brk[s]{\brk{C_H}^2}
        \le
        \bbE \brk[s]{C_H}
        =
        \bbE \brk[s]{\hat c(s_H,a_H)}.
    \]
    Thus we arrived at
    \[
        \bbE \brk[s]{(\wh C)^2}
        \le
        \bbE \brk[s]4{\brk4{\sum_{h=1}^{H-1} C_h + \hat c(s_H,a_H) + \hat c_f(s_{H+1})}^2} 
        +
        \bbE \brk[s]{\hat c(s_H,a_H)},
    \]
    and iterating this argument yields
    \[
        \bbE \brk[s]{(\wh C)^2}
        \le
        \bbE \brk[s]4{\brk4{\sum_{h=1}^{H} \hat c(s_h,a_h) + \hat c_f(s_{H+1})}^2} 
        +
        \bbE \brk[s]4{\sum_{h=1}^{H} \hat c(s_h,a_h)}.
    \]
    Here, the second summand equals $J^m_1(s_1)$ which is at most $4 \costbound$.
    
    Next, for the first summand, we split the time steps into $Q$ blocks as follows.
    We denote by $t_1$ the first time step in which we accumulated a total cost of at least $11 \costbound$ (or $H+1$ if it did not occur), by $t_2$ the first time step in which we accumulated a total cost of at least $11 \costbound$ after $t_1$, and so on up until $t_Q = H+1$. 
    Then, the first block consists of time steps $t_0=1,\dots,t_1-1$, the second block consists of time steps $t_1,\dots,t_2-1$, and so on.
    Since $J_h^m(s) \le 11 \costbound$ we must have $\hat c(s_h,a_h) \le 11 \costbound$ for all $h=1,\ldots,H$ and thus in every such block the total cost is between $11 \costbound$ and $22 \costbound$. Thus, 
    \begin{align*}
        \bbE \brk[s]4{\brk4{\sum_{h=1}^{H} \hat c(s_h,a_h) + \hat c_f(s_{H+1})}^2}  
        & \ge
        \bbE \brk[s]4{\sum_{h=1}^{H} \hat c(s_h,a_h) + \hat c_f(s_{H+1})}^2
        \\
        & =
        \bbE \brk[s]4{\sum_{i=0}^{Q-1} \sum_{h=t_i}^{t_{i+1}-1} \hat c(s_h,a_h) + \hat c_f(s_{H+1})}^2
        \\
        & \ge
        \bbE [11 \costbound Q]^2
        =
        121 \costbound^2 \bbE[Q]^2,
    \end{align*}
    by Jensen's inequality.
    On the other hand,
    \begin{align*}
        \bbE & \brk[s]4{\brk4{\sum_{h=1}^{H} \hat c(s_h,a_h) + \hat c_f(s_{H+1})}^2}
        =
        \bbE \brk[s]4{\brk4{\sum_{h=1}^{H} \hat c(s_h,a_h) + \hat c_f(s_{H+1}) - J_1^m(s_1) + J_1^m(s_1) }^2}
        \\
        & \qquad \le
        2 \bbE \brk[s]4{\brk4{\sum_{h=1}^{H} \hat c(s_h,a_h) + \hat c_f(s_{H+1}) - J_1^m(s_1)}^2}
        + 
        2 J_1^m(s_1)^2 \\
        & \qquad \le
        2 \bbE \brk[s]4{ \brk4{ \sum_{i=0}^{Q-1} \sum_{h=t_i}^{t_{i+1}-1} \hat c(s_h,a_h)  - J^m_{t_i}(s_{t_i}) + J^m_{t_{i+1}}(s_{t_{i+1}})}^2}
        +
        32 \costbound^2 \\
        & \qquad \overset{(a)}{=}
        4 \bbE \brk[s]4{\sum_{i=0}^{Q-1} \brk4{\sum_{h=t_i}^{t_{i+1}-1} \hat c(s_h,a_h) - J^m_{t_i}(s_{t_i}) + J^m_{t_{i+1}}(s_{t_{i+1}})}^2}
        +
        32 \costbound^2 \\
        & \qquad \le
        4 \bbE [Q \cdot (33 \costbound)^2] + 32 \costbound^2
        \le
        4356 \costbound^2 \bbE [Q] + 32 \costbound^2.
    \end{align*}
    For (a) we used the fact that $\bbE\brk[s]{\sum_{h=t_i}^{t_{i+1}-1} \hat c(s_h,a_h) - J_{t_i}(s_{t_i}) + J_{t_{i+1}}(s_{t_{i+1}})} = 0$ using the Bellman optimality equations and conditioned on all past randomness up until time $t_i$, and the fact that $t_{i+1}$ is a (bounded) stopping time by the optional stopping theorem, in the following manner,
    \begin{align*}
        \bbE\brk[s]4{\sum_{h=t_i}^{t_{i+1}-1} \hat c(s_h,a_h) - J^m_{t_i}(s_{t_i}) + J^m_{t_{i+1}}(s_{t_{i+1}})}
        & =
        \bbE\brk[s]4{\sum_{h=t_i}^{t_{i+1}-1} \hat c(s_h,a_h) - J^m_h(s_h) + J^m_{h+1}(s_{h+1})}
        \\
        & =
        \bbE\brk[s]4{\sum_{h=t_i}^{t_{i+1}-1} \bbE\brk[s]2{\hat c(s_h,a_h) - J^m_h(s_h) + J^m_{h+1}(s_{h+1}) ~\big|~ s_1,\ldots,s_h}}
        \\
        & =
        \bbE\brk[s]4{\sum_{h=t_i}^{t_{i+1}-1} \hat c(s_h,a_h) + \bbE\brk[s]2{ J^m_{h+1}(s_{h+1}) \mid s_h} - J^m_h(s_h)}
        =
        0.
    \end{align*}
    Thus, we have
    $
        121 \costbound^2 \bbE[Q]^2
        \le
        4356 \costbound^2 \bbE [Q] + 32 \costbound^2,
    $
    and solving for $\bbE[Q]$ we obtain $\bbE[Q] \le 37$, so
    \[
        \bbE \brk[s]4{\brk4{\sum_{h=1}^{H} \hat c(s_h,a_h) + \hat c_f(s_{H+1})}^2} 
        \le 
        2 \cdot 10^5 \costbound^2, 
    \]
    and therefore
    \[
        \bbE \brk[s]{(\wh C)^2}
        \le
        \bbE \brk[s]4{\brk4{\sum_{h=1}^{H} \hat c(s_h,a_h) + \hat c_f(s_{H+1})}^2} 
        +
        \bbE \brk[s]4{\sum_{h=1}^{H} \hat c(s_h,a_h)}
        \le
        2 \cdot 10^5 \costbound^2 + 4 \costbound.
        \qedhere
    \]
\end{proof}

\begin{proof}[Proof of \cref{lem:cost-deviation-from-value-function}]
    Recall that $h_m$ is the last time step before an unknown state-action pair or the goal state were reached (or $H$ if they were not reached) in interval $m$, and let $\geventi{m}$ be the event that the good event of algorithm $\calA$ holds up to the beginning of interval $m$.
    We start by decomposing the sum as follows
    \begin{align*}
        \sum_{m=1}^\numintervals \brk*{\sum_{h=1}^H C^m_h + \hat c_f(s_{H+1}^m) - \hatctg{\pi^m}_1(s_1^m)} \indevent{\geventi{m}}
        & = 
        \sum_{m=1}^\numintervals \brk*{\sum_{h=1}^{h_m} C^m_h + c_f(s^m_{H+1}) \indevent{h_m = H} - \hatctg{\pi^m}_1(s_1^m)} \indevent{\geventi{m}}
        \\
        & \qquad +
        \sum_{m=1}^\numintervals \brk*{\sum_{h=h_m+1}^H C^m_h + \hat c_f(s_{H+1}^m)\indevent{h_m \ne H}} \indevent{\geventi{m}}.
    \end{align*}
    The second term is trivially bounded by $(H + 8 \costbound) |S| |A| \knownthresh \log \frac{M H |S| |A|}{\delta}$ since every state-action pair becomes known after $\knownthresh \log \frac{M H |S| |A|}{\delta}$ visits.
    Next, since
    \begin{align*}
        \bbE \brk[s]*{\brk*{\sum_{h=1}^{h_m} C^m_h + c_f(s^m_{H+1}) \indevent{h_m = H}} \indevent{\geventi{m}} ~\Bigg|~ \trajconcat{m}} 
        & =
        \bbE \brk[s]*{\sum_{h=1}^{h_m} C^m_h + c_f(s^m_{H+1}) \indevent{h_m = H} ~\Bigg|~ \trajconcat{m}} \indevent{\geventi{m}}
        \\
        & \le
        \hatctg{\pi^m}_1(s_1^m) \indevent{\geventi{m}},
    \end{align*}
    the first term is bounded by $\sum_{m=1}^\numintervals X^m$ where
    \[
        X^m 
        = 
        \brk*{\sum_{h=1}^{h_m} C^m_h + c_f(s^m_{H+1}) \indevent{h_m = H} - \bbE \brk[s]*{\sum_{h=1}^{h_m} C^m_h + c_f(s^m_{H+1}) \indevent{h_m = H} ~\Bigg|~ \trajconcat{m}}} \indevent{\geventi{m}}
    \]
    is a martingale difference sequence bounded by $H+8\costbound$ with probability $1$.
    For any fixed $M = m$, by Freedman's inequality (\cref{lemma: freedmans inequality}, we have with probability at least $1 - \frac{\delta}{8 m(m+1)}$,
    \[
        \sum_{m'=1}^m X^{m'}
        \le
        \eta \sum_{m'=1}^m \bbE [ (X^{m'})^2 \mid \trajconcat{m'}] + \frac{\log(8 m(m+1) / \delta)}{\eta}
    \]
    for any $\eta \in (0 , 1 / (H+8\costbound))$.
    By \cref{lem:reach-goal-or-unknown-wp-1/2}, for some universal constant $\alpha > 0$, that
    \[
        \sum_{m'=1}^m\bbE [ (X^{m'})^2 \mid \trajconcat{m'}] \le \alpha m (\costbound^2 + \costbound),
    \]
    and setting $\eta = \min \bigl\{ \sqrt{\frac{\log (8m(m+1) / \delta) }{ (\costbound^2 + \costbound) m}}  ,\frac{ 1 }{ H + 8\costbound} \bigr\}$ obtains
    \begin{align*}
        \sum_{m'=1}^m X^{m'}
        & \le
        O \Bigl( \sqrt{(\costbound^2 + \costbound) m \log \frac{m}{\delta}} + (H + \costbound) \log \frac{m}{\delta} \Bigr).
    \end{align*}
    Taking a union bound on all values of $m=1,2,\ldots$ that the inequality above holds for all such values of $m$ simultaneously with probability at least $1-\delta/8$. In particular, with probability at least $1-\delta/8$, we have
    \begin{align*}
        \sum_{m=1}^\numintervals X^{m}
        & \le
        O \Bigl( \sqrt{(\costbound^2 + \costbound) \numintervals \log \frac{\numintervals}{\delta}} + (H + \costbound) \log \frac{\numintervals}{\delta} \Bigr).
    \end{align*}

    The proof is concluded via a union bound---both Freedman inequality and the good event of $\calA$ hold with probability at least $1 - \frac{3}{8} \delta$, and this implies that $\indevent{\geventi{m}} = 1$ for every $m$.
\end{proof}

\subsection{Proof of Lemma~\ref{lem:bound-on-number-of-intervals}}
\label{sec:bounding-number-of-intervals}

In this section we bound the number of intervals $\numintervals$ with high probability for any admissible algorithm.
%
To that end, we first define the notion of unknown state-action pairs.
A state-action pair is defined as \emph{unknown} if the number of times it was visited is at most $\knownthresh \log \frac{M H |S| |A|}{\delta}$ (and otherwise \emph{known}).

\begin{proof}[Proof of \cref{lem:bound-on-number-of-intervals}]
    Let $\geventi{m}$ be the event that the good event of algorithm $\calA$ holds up to the beginning of interval $m$, and define $X^m$ to be $1$ if an unknown state-action pair or the goal state were reached during interval $m$ (and $0$ otherwise).
    Notice that $\bbE [X^m \indevent{\geventi{m}} \mid \trajconcat{m}] = \bbE [X^m \mid \trajconcat{m}] \indevent{\geventi{m}} \ge \indevent{\geventi{m}} / 2$ by \cref{lem:reach-goal-or-unknown-wp-1/2}.
    Moreover, note that every state-action pair becomes known after $\knownthresh \log \frac{M H |S| |A|}{\delta}$ visits and therefore $\sum_{m=1}^M X^m \indevent{\geventi{m}} \le \sum_{m=1}^M X^m \le K + |S| |A| \knownthresh \log \frac{M H |S| |A|}{\delta}$.
    By \cref{lemma: consequences of optimism and freedman's inequality}, which is a consequence of Freedman's inequality for bounded positive random variables, we have with probability at least $1 - \frac{\delta}{8}$ for all $M \ge 1$ simultaneously
    \[
        \sum_{m=1}^M \bbE [X^m \indevent{\geventi{m}} \mid \trajconcat{m}]
        \le
        2 \sum_{m=1}^M X^m \indevent{\geventi{m}} + 108 \log \frac{M}{\delta}
        \le
        2 K + 110 |S| |A| \knownthresh \log \frac{M H |S| |A|}{\delta}.
    \]
    Using a union bound, this inequality and the good event of $\calA$ both hold with probability at least $1 - \frac{3}{8} \delta$. Then, $\indevent{\geventi{m}} = 1$ for all $m$, and therefore
    \[
        \frac{M}{2}
        \le 
        2 K + 110 |S| |A| \knownthresh \log \frac{M H |S| |A|}{\delta}.
    \]
    Using the fact that $x \le a \log(bx) + c \rightarrow x \le 6a \log (abc) + c$ for $a,b,c\ge 1$, this implies
    \[
        M 
        \le
        4 K + 4\cdot 10^4 |S| |A| \knownthresh \log \frac{K \timebound |S| |A| \knownthresh}{\delta}.
        \qedhere
    \]
\end{proof}

\newpage

\section{Proofs for Section~\ref{sec:finite-horizon-analysis}} \label{sec:ulcvi-proofs}

Since all the proofs in this section refer to the finite-horizon setting (without a connection to SSP), we use the simpler notations $\calM = (S,A,P,H,c,c_f)$ for the MDP, $J^\pi_h(s)$ for the value function of policy $\pi$, and $\costbound \ge \max_{s,h} J^\star_h(s)$ for the upper bound on the value function of the optimal policy.

We define a state-action pair $(s,a)$ to be \emph{known} if it was visited at least $\alpha H^4 \costbound^{-2} |S|$ times (for some universal constant $\alpha > 0$ to be determined later), and otherwise \emph{unknown}.
In addition, we denote by $h_m$ the last time step before an unknown state-action pair was reached (or $H$ if they were not reached).

\subsection{The good event, optimism and pessimism}

Throughout this section we use the notation $a \vee 1$ defined as $\max\{a, 1\}$.
In addition, we define the logarithmic factor $L_m = 3\log(6|S||A|H m/\delta)$.
Define the following events: 
\begin{align*}
    & E^c(m) 
    = 
    \brc*{\forall (s,a):\ |\bar{c}^{m-1}(s,a) -c(s,a)|  \leq b^m_c(s,a)}
    \\
    & E^{cv}(m)
    =
    \brc*{\forall (s,a):\ \abs*{ \sqrt{\overline{\VAR}^{m-1}_{s,a}(C)} -  \sqrt{\VAR_{s,a}(c)} } \leq \sqrt{\frac{12 L_m }{n^{m-1}(s,a)\vee 1}} }\\
    & E^p(m) 
    = 
    \brc*{\forall (s,a,s'):\ |P\br*{s'|s,a} - \bar{P}^{m-1}\br*{s'|s,a}| \le \sqrt{\frac{2P(s'|s,a)L_m}{n^{m-1}(s,a)\vee 1}} + \frac{2L_m}{n^{m-1}(s,a)\vee 1}} 
    \\
    & E^{pv1}(m)
    =
    \brc*{\forall (s,a,h):\ \abs*{\br*{\bar{P}^{m-1}(\cdot | s,a)-P(\cdot | s,a)} \cdot J_{h+1}^*} \leq \sqrt{\frac{2\VAR_{P(\cdot \mid  s,a)}(J^*_{h+1}) L_m }{n^{m-1}(s,a)\vee 1}} + \frac{5 \costbound L_m }{n^{m-1}(s,a)\vee 1}}
    \\
    & E^{pv2}(m)
    =
    \brc*{\forall (s,a,h):\ \abs*{ \sqrt{\VAR_{P(\cdot \mid  s,a )}(J_{h+1}^*)} -  \sqrt{\VAR_{\bar{P}^{m-1}(\cdot \mid  s,a )}(J_{h+1}^*)} } \leq \sqrt{\frac{12 \costbound^2 L_m}{n^{m-1}(s,a)\vee 1}} }
\end{align*}
For brevity, we denote $b^m_{pv1,h}(s,a) = \sqrt{\frac{2\VAR_{P(\cdot \mid  s,a)}(J^*_{h+1})L_m}{n^{m-1}(s,a)\vee 1}} + \frac{5 \costbound L_m}{n^{m-1}(s,a)\vee 1}$.
This good event, which is the intersection of the above events, is the one used in \citet{efroni2021confidence}. 
The following lemma establishes that the good event holds with high probability. 
The proof is supplied in \citet[Lemma 13]{efroni2021confidence} by applying standard concentration results.

\begin{lemma}[The First Good Event]
    \label{lemma: the first good event RL UL}
    Let $\G_1 =\cap_{m\geq 1} E^c(m) \cap_{m\geq 1} E^{cv}(m) \cap_{m\geq 1} E^p(m) \cap_{m\geq 1} E^{pv1}(m) \cap_{m\geq 1} E^{pv2}(m)$ be the basic good event. 
    It holds that $\Pr(\G_1)\geq 1-\frac14\delta$.
\end{lemma}

Under the first good event, we can prove that the value is optimistic using standard techniques.

\begin{lemma}[Upper Value Function is Optimistic, Lower Value Function is Pessimistic] 
    \label{lemma: optimism ucbvi-UL}
    Conditioned on the first good event $\G_1$, it holds that $\underline{J}^m_h(s) \leq  J^*_h(s) \leq J^{\pi^m}_h(s) \leq \bar{J}^m_h(s)$ for every $m=1,2,\dots$, $s \in S$ and $h=1,\dots,H+1$.
\end{lemma}

\begin{proof}
    Since $J^*_h(s) \leq J^\pi_h(s)$ for any policy $\pi$, we only need to prove the leftmost and rightmost inequalities of the claim. 
    We prove this result via induction.

    \paragraph{Base case, the claim holds for $h=H+1$.} 
    Since we assume the terminal costs are known, for any $s \in S$,  
    \[
        \underline{J}^m_{H+1}(s)
        = 
        J^*_{H+1}(s)
        = 
        J^{\pi^m}_{H+1}(s) 
        = 
        \bar{J}^m_{H+1}(s)
        =  
        c_f(s).
    \]

    \paragraph{Induction step, prove for $h \in [H]$ assuming the claim holds for all $h+1 \leq h' \leq H+1$.}

    \paragraph{Leftmost inequality, optimism.}
    Let $a^*(s)\in \argmin_{a \in A} Q^{*}_h(s,a)$, then
    \begin{align}
        \label{eq: optimism base case ucbvi rel 2 UL}
        J^*_h(s) - \underline{J}^m_{h}(s) 
        =   
        Q^{*}_h(s,a^*(s)) - \max\brc*{\min_{a \in A} \underline{Q}^m_{h}(s,a),0}.
    \end{align}
    Assume that $\min_a \bar{Q}^m_{h}(s,a) > 0$ (otherwise, the inequality is satisfied). 
    Then,
    \begin{align}
        \nonumber
        \eqref{eq: optimism base case ucbvi rel 2 UL} 
        & \geq 
        Q^{*}_h(s,a^*(s)) - \underline{Q}^m_{h}(s,a^*(s))
        \\
        \nonumber
        & = 
        c(s,a^*(s)) - \bar c^{m-1}(s,a^*(s)) + b^m_c(s,a^*(s)) + b^m_p(s,a^*(s))  
        \\
        \nonumber
        & \quad + 
        (P- \bar{P}^{m-1})(\cdot \mid  s,a^*(s)) \cdot J^*_{h+1} + \E_{\bar{P}^{m-1}(\cdot \mid s,a^*(s))}[ \underbrace{J^{*}_{h+1}(s') - \underline{J}^m_{h+1}(s')}_{\ge 0\ \mathrm{Induction\ hypothesis}}]
        \\
        \label{eq: optimism base case ucbvi rel 3 UL}
        & \geq 
        -b^m_{pv1,h}(s,a^*(s)) + b^m_p(s,a^*(s)), 
    \end{align}
    where the last relation holds since the events $\cap_m E^{pv1}(m)$ and $\cap_m E^{c}(m)$ hold.
    We now analyze this term.
    \begin{align*}
        \eqref{eq: optimism base case ucbvi rel 3 UL} 
        & = 
        -b^m_{pv1,h}(s,a^*(s)) + b^m_p(s,a^*(s))
        \\
        & \overset{(a)}{\geq}
        - \sqrt{ \frac{2 \VAR_{P(\cdot \mid  s,a^*(s))}(J^*_{h+1})L_m}{n^{m-1}(s,a^*(s)) \vee1}} - \frac{ 5 \costbound L_m }{n^{m-1}(s,a^*(s)) \vee 1}
        \\
        &\quad + \sqrt{\frac{2\VAR_{\bar{P}^{m-1}(\cdot \mid s,a^*(s))}(\underline{J}^m_{h+1})L_m}{n^{m-1}(s,a^*(s)) \vee 1}} + \frac{17H^3 \costbound^{-1} L_m}{n^{m-1}(s,a^*(s))\vee 1}+ \frac{\costbound}{16H^2} \E_{\bar{P}^{m-1}(\cdot \mid s,a)}\brs*{J^*_{h+1}(s') - \underline{J}^m_{h+1}(s')} 
        \\
        & \ge -\sqrt{2L_m}\frac{\sqrt{\VAR_{P(\cdot \mid s,a^*(s))}(J^*_{h+1})} -\sqrt{\VAR_{\bar{P}^{m-1}(\cdot \mid  s,a^*(s))}(\underline{J}^m_{h+1})}}{\sqrt{n^{m-1}(s,a^*(s)) \vee1}}
        \\
        & \quad + \frac{\costbound}{16H^2}  \E_{\bar{P}^{m-1}(\cdot \mid s,a)}\brs*{J^*_{h+1}(s') - \underline{J}^m_{h+1}(s')}+ \frac{13H^3 \costbound^{-1} L_m}{n^{m-1}(s,a^*(s))\vee 1}
        \\
        & \overset{(b)}{\geq}
        - \frac{\costbound}{16H^2}  \E_{\bar{P}^{m-1}(\cdot \mid s,a)}\brs*{J^*_{h+1}(s') - \underline{J}^m_{h+1}(s')} - \frac{13 H^2 L_m}{n^{m-1}(s,a^*(s)) \vee1}
        \\
        &\quad +
        \frac{\costbound}{16H^2}\E_{\bar{P}^{m-1}(\cdot \mid s,a)}\brs*{J^*_{h+1}(s') - \underline{J}^m_{h+1}(s')} + \frac{13H^3 \costbound^{-1} L_m}{ n^{m-1}(s,a)\vee 1}  
        \ge
        0, 
    \end{align*}
    where $(a)$ holds by plugging the definition of the bonuses $b_{pv1,h}^m$ and $b_p^m$ (recall \cref{eq:bonus-definitions}), as $|S| \geq 1$ by assumption, and by the induction hypothesis ($\bar{J}^{m}_{h+1}(s) \ge J^*_{h+1}(s)$). $(b)$ holds by Lemma~\ref{lemma: variance diff is upper bounded by value difference} while setting $\alpha= 16H^2\costbound^{-1}$ and bounding $(5+ \alpha/2) \costbound \leq 13H^2$.
    Combining all the above we conclude the proof of the rightmost inequality since $J^*_h(s) - \underline{J}^m_h(s) \geq \eqref{eq: optimism base case ucbvi rel 2 UL} \geq \eqref{eq: optimism base case ucbvi rel 3 UL} \ge 0$.

    \paragraph{Rightmost inequality, pessimism.} 
    The following relations hold.
    \begin{align}
        \label{eq: optimism base case ucbvi UL rel 2, pessimsm}
        J^{\pi^m}_h(s) - \bar{J}^m_h(s) 
        =   
        Q^{\pi^m}_h(s,\pi^m_h(s))- \min \brc*{ \bar{Q}^m_{h}(s,\pi^m_h(s)),H }.
    \end{align}
    Assume that $\bar{Q}^m_{h}(s,\pi^m_h(s))<H$ (otherwise, the claim holds). 
    Then,
    \begin{align}
        \nonumber
        \eqref{eq: optimism base case ucbvi UL rel 2, pessimsm} 
        & =  
        Q^{\pi^m}_h(s,\pi^m_h(s))- \bar{Q}^m_{h}(s,\pi^m_h(s)) 
        \\
        \nonumber
        & = 
        c(s,\pi^m_h(s)) - \bar c^{m-1}(s,\pi^m_h(s)) - b^m_c(s,\pi^m_h(s)) - b^m_p(s,\pi^m_h(s))
        \\
        \nonumber
        & \quad + (P-\bar{P}^{m-1})(\cdot \mid  s,\pi^m_h(s)) \cdot J^{\pi^m}_{h+1} + \E_{\bar{P}^{m-1}(\cdot \mid s,\pi^m_h(s))}[ \underbrace{J^{\pi^m}_{h+1}(s') - \bar{J}^m_{h+1}(s')}_{\leq 0\ \mathrm{Induction\ hypothesis}}]
        \\
        \label{eq: optimism base case ucbvi UL rel 3, pessimsm}
        & \le 
        - b^m_p(s,\pi^m_h(s)) 
        + 
        (P - \bar{P}^{m-1})(\cdot \mid s,\pi^m_h(s)) \cdot J^{\pi^m}_{h+1}.
    \end{align}
    We now focus on the last term. 
    Observe that
    \begin{align*}
        (P - \bar{P}^{m-1})(\cdot \mid s,\pi^m_h(s)) & \cdot J^{\pi^m}_{h+1} 
        = 
        (P - \bar{P}^{m-1})(\cdot \mid  s,\pi^m_h(s)) \cdot J^{*}_{h+1} + (P-\bar{P}^{m-1})(\cdot \mid  s,\pi^m_h(s)) \cdot (J^{\pi^m}_{h+1} - J^*_{h+1})
        \\
        & \leq 
        b^m_{pv1,h}(s,\pi^m_h(s))+(P - \bar{P}^{m-1})(\cdot \mid  s,\pi^m_h(s)) \cdot (J^{\pi^m}_{h+1} - J^*_{h+1}) 
        \tag{$\cap_m E^{pv1}(m)$ holds}
        \\
        & \overset{(a)}{\leq}  
        b^m_{pv1,h}(s,\pi^m_h(s)) + \frac{36H^3 \costbound^{-1} |S| L_m}{n^{m-1}(s,\pi^m_h(s))\vee 1} + \frac{\costbound}{32H^2}\E_{\bar{P}^{m-1}(\cdot \mid s,\pi^m_h(s))}\brs*{(J^{\pi^m}_{h+1}-J^*_{h+1})(s')} 
        \\
        & \overset{(b)}{\leq}  
        b^m_{pv1,h}(s,\pi^m_h(s)) + \frac{36H^3 \costbound^{-1} |S|L_m}{n^{m-1}(s,\pi^m_h(s))\vee 1} + \frac{\costbound}{32H^2}\E_{\bar{P}^{m-1}(\cdot \mid s,\pi^m_h(s))}\brs*{(\bar{J}^m_{h+1}-\underline{J}^m_{h+1})(s') } 
        \\
        & \overset{(c)}{\leq} 
        \sqrt{\frac{2\VAR_{P(\cdot \mid  s,\pi^m_h(s))}(J^*_{h+1})L_m}{n^{m-1}(s,\pi^m_h(s))\vee 1 }}+ \frac{41H^3 \costbound^{-1} |S|L_m}{n^{m-1}(s,\pi^m_h(s))\vee 1} 
        \\
        & \qquad + \frac{\costbound}{32H^2}\E_{\bar{P}^{m-1}(\cdot \mid s,\pi^m_h(s))}\brs*{(\bar{J}_{t-1,h+1} -\underline{J}^m_{h+1})(s')},
    \end{align*}
    where $(a)$ holds by applying \cref{lemma: transition different to next state expectation} while setting $\alpha=32H^2 \costbound^{-1}, C_1=2,C_2=2$ and bounding $2C_2+ \alpha |S| C_1/2\leq 36H^2 \costbound^{-1} |S|$ (assumption holds since $\cap_m E^{p}(m)$ holds), $(b)$ holds by the induction hypothesis, and $(c)$ holds by plugging in $b^m_{pv1,h}$.
    Plugging this back into~\eqref{eq: optimism base case ucbvi UL rel 3, pessimsm} and plugging the explicit form of the bonus $b^m_p(s,a)$ we get
    \begin{align*}
        \eqref{eq: optimism base case ucbvi UL rel 3, pessimsm} 
        & \le -\sqrt{2L_m}\frac{\sqrt{\VAR_{\bar{P}^{m-1}(\cdot \mid  s,\pi^m_h(s))}(\underline{J}^m_{h+1})} - \sqrt{\VAR_{P(\cdot \mid  s,\pi^m_h(s))}(J^*_{h+1})}}{\sqrt{n^{m-1}(s,\pi^m_h(s))\vee 1}} 
        \\
        & \quad -
        \frac{21H^3 \costbound^{-1} |S| L_m}{n^{m-1}(s,\pi^m_h(s))\vee 1} - \frac{\costbound}{32H^2}\E_{\bar{P}^{m-1}(\cdot \mid s,\pi^m_h(s))}\brs*{\bar{J}^m_{h+1}(s') - \underline{J}^m_{h+1}(s')}
        \\
        & \le \frac{\costbound}{32H^2}\E_{\bar{P}^{m-1}(\cdot \mid s,\pi^m_h(s))}\brs*{J^*_{h+1}(s')-\underline{J}^m_{h+1}(s')} + \frac{21H^3 \costbound^{-1} L_m}{n^{m-1}(s,\pi^m_h(s))}
        \\
        & \quad -  
        \frac{\costbound}{32H^2}\E_{\bar{P}^{m-1}(\cdot \mid s,\pi^m_h(s))}\brs*{\bar{J}^m_{h+1}(s') - \underline{J}^m_{h+1}(s')} - \frac{21H^3 \costbound^{-1} |S| L_m}{n^{m-1}(s,\pi^m_h(s))}
        =  
        0,
    \end{align*}
    where the last inequality holds by Lemma~\ref{lemma: variance diff is upper bounded by value difference} while setting $\alpha=32H^2 \costbound^{-1}$ and bounding $(5+ \alpha/2) \costbound\leq 21H^3 \costbound^{-1}$. Combining all the above we concludes the proof as
    \begin{align*}
        J^{\pi^m}_h(s)- \bar{J}^m_h(s)\leq\eqref{eq: optimism base case ucbvi UL rel 2, pessimsm}\leq \eqref{eq: optimism base case ucbvi UL rel 3, pessimsm}
        \leq 
        0
        . 
        &\qedhere
    \end{align*}
\end{proof}

Finally, using similar techniques to \citet{efroni2021confidence}, we can prove an additional high probability bounds which hold alongside the basic good event $\G_1$.

\begin{lemma}[The Good Event]
    \label{lem:ULCVI-good-event}
    Let $\G_1$ be the event defined in \cref{lemma: the first good event RL UL}, and define the following random variables.
    \begin{align*}
        Y^m_{1,h} 
        & = 
        \bar{J}^m_h(s^m_h) - \underline{J}^m_h(s^m_h) 
        \\
        Y^m_{2,h} 
        & = 
        \VAR_{P(\cdot \mid s^m_h,a^m_h)}( J^{\pi^m}_{h+1})
        \\
        Y^m_{3}
        & =
        \br*{ \sum_{h=1}^H c(s^m_h,a^m_h) + c_f(s^m_{h+1})}^2
        \\
        Y^m_{4}
        & =
        \br*{ \sum_{h=1}^{h_m} c(s^m_h,a^m_h) + c_f(s^m_{h+1}) \indevent{h_m = H}}^2
        \\
        Y^m_{5}
        & =
        \sum_{h=1}^{h_m} c(s^m_h,a^m_h) + c_f(s^m_{h+1}) \indevent{h_m = H}.
    \end{align*}
    The second good event is the intersection of two events $\G_2 = E^{OP} \cap  E^{\VAR} \cap E^{Sec1} \cap E^{Sec2} \cap E^{cost}$ defined as follows.
    \begin{align*}
        & E^{OP}
        =
        \brc*{\forall h\in[H], M \geq 1:\ \sum_{m=1}^M \E[Y^m_{1,h} \mid \trajconcat{m}_h]\leq 68H^2 L_M + \br*{1+\frac{1}{4H}} \sum_{m=1}^M Y^m_{1,h}}
        \\
        & E^{\VAR}
        = 
        \brc*{ \forall M \geq 1:\  \sum_{m=1}^M \sum_{h=1}^H Y^m_{2,h}\leq  16 H^3 L_M + 2\sum_{m=1}^M \sum_{h=1}^H\E[Y^m_{2,h}|\trajconcat{m}]}
        \\
        & E^{Sec1}
        =
        \brc*{\forall M \ge 1:\ \sum_{m=1}^M \E[Y^m_{3} \mid \trajconcat{m}]\leq 68H^4 L_M + 2 \sum_{m=1}^M Y^m_{3}}
        \\
        & E^{Sec2}
        = 
        \brc*{ \forall M \geq 1:\  \sum_{m=1}^M Y^m_{4}\leq 16 H^4 L_M + 2\sum_{m=1}^M \E[Y^m_{4} \mid \trajconcat{m}]}
        \\
        & E^{cost}
        = 
        \brc*{ \forall M \geq 1:\  \sum_{m=1}^M Y^m_{5} \leq 8 H L_M + 2 \sum_{m=1}^M \E[Y^m_{5} \mid \trajconcat{m}]}.
    \end{align*}
    Then, the good event $\G = \G_1\cap \G_2$ holds with probability at least $1-\delta$.
\end{lemma}

\begin{proof}
    {\bf Event $E^{OP}$.}
    Fix $h$ and $M$. 
    We start by defining the random variable $W^m=\indevent{\bar{J}^m_h(s) - \underline{J}^m_h(s)\geq 0 \  \forall h\in\brs*{H},s\in S}$.
    Observe that $Y^m_h$ is $\trajconcat{m}_h$ measurable and also notice that $W^m$ is $\trajconcat{m}$ measurable, as both $\pi^m$ and $\bar{J}^m_h$ are $\trajconcat{m}$-measurable. 
    Finally, define $\tilde{Y}^m = W^m Y^m_h$.
    Importantly, notice that $\tilde{Y}^m \in\brs*{0,2H}$ almost surely, by definition of $W^m$ and since $\bar{J}^m_h(s),\underline{J}^m_h(s)\in [0,2H]$ by the update rule. 
    Thus, using \cref{lemma: consequences of optimism and freedman's inequality} with $C=2H\geq 1$, we get
    \[
        \sum_{m=1}^M \E[\tilde{Y}^m_h \mid \trajconcat{m}_h]
        \leq 
        \br*{1+\frac{1}{4H}} \sum_{m=1}^M \tilde{Y}^m_h  + 68 H^2 \log\frac{2HM(M+1)}{\delta},
    \]
    with probability greater than $1-\delta$, and since $W^m$ is $\trajconcat{m}$-measurable, we can write
    \begin{align}
        \label{eq: good event relation 1}
        \sum_{m=1}^M W^m\E[Y^m_h|\trajconcat{m}_h]
        \leq 
        \br*{1+\frac{1}{4H}} \sum_{m=1}^M W^m Y^m_h  + 68H^2 \log\frac{2HM(M+1)}{\delta}.
    \end{align}
    Importantly, notice that under $\G_1$, it holds that $W^m \equiv1$ (by \Cref{lemma: optimism ucbvi-UL}). 
    Therefore, applying the union bound and setting $\delta= \delta/(2HM(M+1))$ we get
    \begin{align*}
        \Pr(&\overline{E^{O}}\cap \G_1) \le
        \\
        & \leq 
        \sum_{h=1}^H\sum_{M=1}^\infty \Pr\br*{\brc*{\sum_{m=1}^M \E[Y^m_h |\trajconcat{m}_h]\geq \br*{1+\frac{1}{4H}} \sum_{m=1}^M Y^m_h + 68H^2 \log\frac{2HM(M+1)}{\delta}}\cap\G_1}
        \\
        & = 
        \sum_{h=1}^H\sum_{M=1}^\infty \Pr\br*{\brc*{\sum_{m=1}^M W^m \E[Y^m_h |\trajconcat{m}_h]\geq \br*{1+\frac{1}{4H}} \sum_{m=1}^M W^m Y^m_h + 68H^2 \log\frac{2HM(M+1)}{\delta}}\cap\G_1}
        \\
        & \le 
        \sum_{h=1}^H\sum_{M=1}^\infty \Pr\br*{\sum_{m=1}^M W^m \E[Y^m_h |\trajconcat{m}_h]\geq \br*{1+\frac{1}{4H}} \sum_{m=1}^M W^m Y^m_h + 68H^2 \log\frac{2HM(M+1)}{\delta}}
        \\
        & \leq 
        \sum_{h=1}^H\sum_{M=1}^\infty \frac{\delta}{2HM(M+1)}=\delta/2,
    \end{align*}
    where the first relation is by a union bound, the second relation follows because $W^m\equiv1$ under $\G_1$, and the last relation is by \eqref{eq: good event relation 1}.
    Finally, we have 
    \[
        \Pr(\overline{\G}) 
        \le 
        \Pr(\overline{\G_2}\cap\G_1) + 2\Pr(\overline{\G_1}) 
        \le 
        \frac{\delta}{2} + \frac{2\delta}{4} = 
        \delta.
    \]
    Replacing $\delta\to\delta/5$ implies that $\Pr(\overline{E^{OP}}\cap \G_1)\le\frac{\delta}{10}$.

    {\bf Event $E^{\VAR}$.} 
    Fix $h\in [H]$. 
    Observe that $Y^m_{2,h}$ is $\trajconcat{m}$ measurable and that $0\leq Y^m_{2,h}\leq 4 H^2$. 
    Applying the second statement of \Cref{lemma: consequences of optimism and freedman's inequality} we get that
    \[
        \sum_{m=1}^M Y^m_{2,h} 
        \leq  
        2 \sum_{m=1}^M \E[Y^m_{2,h}|\trajconcat{m}] + 16 H^2 \log\frac{1}{\delta}.
    \]
    By taking union bound, as in the proof of the first statement of the lemma on all $h\in [H]$ and summing over $h\in [H]$, we get that with probability at least $1-\delta/10$ for all $M\geq1$ it holds that
    \[
        \sum_{m=1}^M \sum_{h=1}^H  Y^m_{2,h} 
        \leq  
        2\sum_{m=1}^M\sum_{h=1}^H \E[Y^m_{2,h}|\trajconcat{m}] + 16 H^3 L_M. 
    \]
    
    {\bf Event $E^{Sec1}$.}
    Observe that $Y^m_{3}$ is $\trajconcat{m}$ measurable and that $0\leq Y^m_{3}\leq 4 H^2$.
    Applying the first statement of \Cref{lemma: consequences of optimism and freedman's inequality} we get that
    \[
        \sum_{m=1}^M \E[Y^m_{3}|\trajconcat{m}] 
        \leq  
        2 \sum_{m=1}^M Y^m_{3} + 50 H^4 \log\frac{1}{\delta}.
    \]
    By taking union bound we get that with probability at least $1-\delta/10$ the event holds.
    
    {\bf Event $E^{Sec2}$.}
    Observe that $Y^m_{4}$ is $\trajconcat{m}$ measurable and that $0\leq Y^m_{4}\leq 4 H^2$.
    Applying the second statement of \Cref{lemma: consequences of optimism and freedman's inequality} we get that
    \[
        \sum_{m=1}^M Y^m_{4} 
        \leq  
        2 \sum_{m=1}^M \E[Y^m_{4}|\trajconcat{m}] + 16 H^2 \log\frac{1}{\delta}.
    \]
    By taking union bound we get that with probability at least $1-\delta/10$ the event holds.
    
    {\bf Event $E^{cost}$.}
    Observe that $Y^m_{5}$ is $\trajconcat{m}$ measurable and that $0\leq Y^m_{5}\leq 2 H$.
    Applying the second statement of \Cref{lemma: consequences of optimism and freedman's inequality} we get that
    \[
        \sum_{m=1}^M Y^m_{5} 
        \leq  
        2 \sum_{m=1}^M \E[Y^m_{5}|\trajconcat{m}] + 8 H \log\frac{1}{\delta}.
    \]
    By taking union bound we get that with probability at least $1-\delta/10$ the event holds.
    
    {\bf Combining all the above.}
    We bound the probability of $\overline{G}$ as follows:
    \[
        \Pr(\overline{\G}) 
        \le 
        \Pr(\overline{\G_1}) + \Pr(\overline{E^{OP}}\cap\G_1) + \Pr(\overline{E^{\VAR}}) + \Pr(\overline{E^{Sec1}}) +\Pr(\overline{E^{Sec2}}) + \Pr(\overline{E^{cost}})  
        \le 
        \frac{\delta}{2} + 5 \cdot \frac{\delta}{10}
        = 
        \delta.
    \]
\end{proof}

\subsection{ULCVI is admissible}
\label{sec:ULCVI-admissible}

By the definition of the algorithm and its regret bound in \cref{thm:ulcvi-guarantees}, it is clear that properties 1,2,3 of the admissible algorithm definition hold.
Thus, it remains to show property 4 by bounding $\omega_{\text{ULCVI}}$.
In order to show that $\omega_{\text{ULCVI}} = O(H^4 \costbound^{-2} |S|)$, we need to show that if the number of visits to $(s,a)$ is at least $\alpha H^4 \costbound^{-2} |S| \log \frac{MH|S||A|}{\delta}$ (for a large enough universal constant $\alpha > 0$) then $\lVert P(\cdot \mid s,a) - \wt P_t(\cdot \mid s,a) \rVert_1 \le 1/(18H)$ and $|c(s,a) - \tilde c^t_h(s,a)| \le \costbound/H$ (under the good event), where $\wt P,\tilde c$ are the estimations used by the algorithm to compute its optimistic $Q$-function (i.e., these are the empirical transition estimate and the empirical cost estimate plus the bonus).

Indeed, by event $\cap_{m>0} E^p(m)$, 
\begin{align*}
    \lVert P(\cdot \mid s,a) - \wt P(\cdot \mid s,a) \rVert_1
    & =
    \lVert P(\cdot \mid s,a) - \bar P(\cdot \mid s,a) \rVert_1
    \\
    & \le
    \sqrt{\frac{2 |S| \log \frac{16M^3 H |S|^2 |A|}{\delta}}{n(s,a)}} + \frac{2 |S| \log \frac{16M^3 H |S|^2 |A|}{\delta}}{n(s,a)}
    \\
    & \le
    \frac{4 \costbound}{\sqrt{\alpha} H^2} + \frac{16 \costbound^2}{\alpha H^4}
    \le
    \frac{1}{18 H},
\end{align*}
for $\alpha > 5800$, where the first inequality holds by Jensen inequality and since event $\cap_{m>0} E^p(m)$ holds.
By the definition of the exploration bonuses we have
\begin{align*}
    |c(s,a) & - \tilde c_h(s,a)|
    \le
    |c(s,a) - \bar c(s,a)| + b_c(s,a) + b_p(s,a)
    \\
    & \le
    3 \sqrt{\frac{2 \costbound^2 \log \frac{16M^3 H |S|^2 |A|}{\delta}}{n(s,a)}} + \frac{72 H^3 \costbound^{-1} |S| \log \frac{16M^3 H |S|^2 |A|}{\delta}}{n(s,a)} + \frac{\costbound \max_{s'} \bar J_{h+1}(s') - \underline{J}_{h+1}(s')}{16 H^2}
    \\
    & \le
    \frac{12 \costbound^2}{\sqrt{\alpha}H^2} + \frac{800 \costbound}{\alpha H} + \frac{\costbound}{16 H}
    \le
    \frac{\costbound}{H},
\end{align*}
for $\alpha > 5800$.

Finally, note that although our algorithm does not update the policy in the beginning of every episode (only when the number of visits to some state-action pair is doubled), this only implies that the constant $\alpha$ needs to be doubled.

\subsection{Proof of Theorem~\ref{thm:ulcvi-guarantees}}

As in the proof of UCBVI, before establishing the proof of \cref{thm:ulcvi-guarantees} we establish the following key lemma that bounds the on-policy errors at time step $h$ by the on-policy errors at time step $h+1$ and additional additive terms. 
Given this result, the analysis follows with relative ease.

\begin{lemma}[ULCBVI, Key Recursion Bound]
    \label{lemma: key recursion bound UL}
    Conditioning on the good event $\G$, the following bound holds for all $h\in [H]$. 
    \begin{align*}
        \sum_{m=1}^M \bar{J}^m_h(s^m_h)-\underline{J}^m_h(s^m_h)
        & \leq  
        68H^2 L_M +\sum_{m=1}^M \frac{310 H^3 \costbound^{-1} |S| L_m}{ n^{m-1}(s^m_h,a^m_h)\vee 1} 
        + \sum_{m=1}^M 4 \sqrt{L_m} \frac{\sqrt{ c(s^m_h,a^m_h)}}{\sqrt{n^{m-1}(s^m_h,a^m_h)\vee 1}} 
        \\
        & \quad + 
        \sum_{m=1}^M 2\sqrt{2L_m}\frac{\sqrt{\VAR_{P(\cdot \mid s^m_h,a^m_h)}(J^{\pi^m}_{h+1})}}{\sqrt{n^{m-1}(s^m_h,a^m_h)\vee 1}} + \br*{1+ \frac{1}{2H}}^2\sum_{m=1}^M \brk1{\bar{J}^m_{h+1}(s^m_{h+1}) -\underline{J}^m_{h+1}(s^m_{h+1})}
        .
    \end{align*}
\end{lemma}

\begin{proof}
    We bound each of the terms in the sum as follows.
    \begin{align}
        \nonumber
        \bar{J}^m_h(s^m_h) - \underline{J}^m_h(s^m_h)
        & =   
        2b^m_c(s^m_h,a^m_h) + 2 b^m_p(s^m_h,a^m_h) + \E_{\bar{P}^{m-1}(\cdot \mid  s^m_h,a^m_h)}[  \bar{J}^m_{h+1}(s^m_{h+1}) - \underline{J}^m_{h+1}(s^m_{h+1})]
        \\
        \nonumber
        & = 2b^m_c(s^m_h,a^m_h) + 2 b^m_p(s^m_h,a^m_h) 
        \\
        \nonumber
        & \quad + 
        \E_{P(\cdot \mid  s^m_h,a^m_h)} [ \bar{J}^m_{h+1}(s^m_{h+1}) - \underline{J}^m_{h+1}(s^m_{h+1})] + (\bar{P}^{m-1}-P)(\cdot |s^m_h,a^m_h) \cdot \br*{\bar{J}^m_{h+1} - \underline{J}^m_{h+1}}
        \\
        \nonumber
        & \leq  
       2b^m_c(s^m_h,a^m_h) + 2 b^m_p(s^m_h,a^m_h)
       \\
       \label{eq: central theorem UL RL relation 1}
       & \quad +
       \frac{8H^2 |S| L_m}{n^{m-1}(s^m_h,a^m_h)\vee 1}+ \br*{1+\frac{1}{4H}} \E_{P(\cdot \mid  s^m_h,a^m_h)} [ \bar{J}^m_{h+1}(s^m_{h+1}) - \underline{J}^m_{h+1}(s^m_{h+1})],
    \end{align}
    where the last relation holds by \cref{lemma: transition different to next state expectation} which upper bounds
    \begin{align*}
        (\bar{P}^{m-1}-P)(\cdot |s^m_h,a^m_h) \cdot \br*{\bar{J}^m_{h+1} - \underline{J}^m_{h+1}} \leq \frac{8H^2 |S| L_m}{n^{m-1}(s^m_h,a^m_h)\vee 1}+ \frac{1}{4H} \E_{P(\cdot \mid  s^m_h,a^m_h)} [ \bar{J}^m_{h+1}(s^m_{h+1}) - \underline{J}^m_{h+1}(s^m_{h+1})]
    \end{align*}
    by setting $\alpha=4H,C_1=C_2=2$ and bounding $H L_m ( 2C_2+ \alpha |S| C_1/2)\leq 8H^2 |S| L_m$ (the assumption of the lemma holds since the event $\cap_m E^p(m)$ holds). 
    Taking the sum over $m\in [M]$ we get that
    \begin{align}
        \nonumber
        \sum_{m=1}^M \bar{J}^m_h(s^m_h)-\underline{J}^m_h(s^m_h)
        & \leq 
        \sum_{m=1}^M 2b^m_c(s^m_h,a^m_h) + \sum_{m=1}^M 2 b^m_p(s^m_h,a^m_h)
        \\
        \label{eq: central theorem UL RL relation 12}
        & \quad + 
        \sum_{m=1}^M\frac{8H^2 |S| L_m}{n^{m-1}(s^m_h,a^m_h)\vee 1}+\sum_{m=1}^M \br*{1+\frac{1}{4H}} \E_{P(\cdot \mid  s^m_h,a^m_h)} [ \bar{J}^m_{h+1}(s^m_{h+1}) - \underline{J}^m_{h+1}(s^m_{h+1})].
    \end{align}
    The first sum is bounded in  \cref{lemma: bound on bonus bc UL RL} by
    \[
        \sum_{m=1}^M  b^m_c(s^m_h,a^m_h)
        \leq 
        \sum_{m=1}^M \sqrt{ \frac{2 c(s^m_h,a^m_h)L_m }{n^{m-1}(s^m_h,a^m_h)\vee 1}} +  \sum_{m=1}^M \frac{10L_m}{ n^{m-1}(s^m_h,a^m_h)\vee 1},
    \]
   and the second sum is bounded in \cref{lemma: bound on bonus bp UL RL} by
    \begin{align*}
        \sum_{m=1}^M b^m_p(s^m_h,a^m_h)
        & \leq 
        \sum_{m=1}^M \frac{139H^3 \costbound^{-1} |S| L_m}{ n^{m-1}(s^m_h,a^m_h)\vee 1} + \sum_{m=1}^M \sqrt{2L_m}\frac{\sqrt{\VAR_{P(\cdot \mid s^m_h,a^m_h)}(J^{\pi^m}_{h+1})}}{\sqrt{n^{m-1}(s^m_h,a^m_h)\vee 1}}
        \\
        & \quad + 
        \frac{1}{8H}\sum_{m=1}^M \E_{P(\cdot \mid  s^m_h,a^m_h)} [ \bar{J}^m_{h+1}(s^m_{h+1}) - \underline{J}^m_{h+1}(s^m_{h+1})].
    \end{align*}
    Plugging this into~\eqref{eq: central theorem UL RL relation 12} and rearranging the terms we get
    \begin{align*}
        \sum_{m=1}^M \bar{J}^m_h(s^m_h) - \underline{J}^m_h(s^m_h) 
        & \leq 
        \sum_{m=1}^M \frac{2 \sqrt{2 c(s^m_h,a^m_h) L_m}}{\sqrt{n^{m-1}(s^m_h,a^m_h)\vee 1}} + \sum_{m=1}^M 2\sqrt{2L_m}\frac{\sqrt{\VAR_{P(\cdot \mid s^m_h,a^m_h)}(J^{\pi^m}_{h+1})}}{\sqrt{n^{m-1}(s^m_h,a^m_h)\vee 1}}
        \\
        & \quad + 
        \sum_{m=1}^M \frac{286H^3 \costbound^{-1} |S| L_m}{ n^{m-1}(s^m_h,a^m_h)\vee 1} + \br*{1+ \frac{1}{2H}}\sum_{m=1}^M \E_{P(\cdot \mid  s^m_h,a^m_h)} [ \bar{J}^m_{h+1}(s^m_{h+1}) - \underline{J}^m_{h+1}(s^m_{h+1})]
        \\
        &\leq  
        68H^2 L_M + \sum_{m=1}^M \frac{2 \sqrt{2L_m}}{\sqrt{n^{m-1}(s^m_h,a^m_h)\vee 1}} + \sum_{m=1}^M\frac{286H^3 \costbound^{-1} |S| L_m}{ n^{m-1}(s^m_h,a^m_h)\vee 1}
        \\
        & \quad + \sum_{m=1}^M 2\sqrt{2L_m}\frac{\sqrt{\VAR_{P(\cdot \mid s^m_h,a^m_h)}(J^{\pi^m}_{h+1})}}{\sqrt{n^{m-1}(s^m_h,a^m_h)\vee 1}} +
        \br*{1+ \frac{1}{2H}}^2\sum_{m=1}^M \bar{J}^m_{h+1}(s^m_{h+1}) -  \underline{J}^m_{h+1}(s^m_{h+1}),
    \end{align*}
    where the last inequality follows since the second good event holds.
\end{proof}

\begin{proof}[Proof of \cref{thm:ulcvi-guarantees}]
    Start by conditioning on the good event which holds with probability greater than $1-\delta$. 
    Applying the optimism-pessimism of the upper and lower value function we get
    \begin{align}
        \label{eq: central thm UL RL 1 relation}
        \sum_{m=1}^{M} J_{1}^{\pi^m}(s^m_1) - J_{1}^*(s^m_1)   
        \leq 
        \sum_{m=1}^{M} \bar{J}^m_1(s^m_1) - \underline{J}^m_1(s^m_1).
    \end{align}
    Iteratively applying \cref{lemma: key recursion bound UL} and bounding the exponential growth by $\br*{1+\frac{1}{2H}}^{2H}\leq e\leq 3$, the following upper bound on the cumulative regret is obtained.
    \begin{align}
        \nonumber
        \eqref{eq: central thm UL RL 1 relation} 
        & \leq 
        204H^3 \costbound^{-1} L_M +  \sum_{m=1}^M \sum_{h=1}^H\frac{ 930 H^3 \costbound^{-1} |S|L_m}{n^{m-1}( s^m_h,a^m_h)\vee 1} 
        \\
        \label{eq: rl final bound relation 2 UL}
        & \quad + 
        \sum_{m=1}^M\sum_{h=1}^H    \frac{12 \sqrt{ c(s^m_h,a^m_h)L_m}}{\sqrt{n^{m-1}(s^m_h,a^m_h)\vee 1}} +9\sum_{m=1}^M \sum_{h=1}^H \frac{\sqrt{L_m\VAR_{P(\cdot \mid s^m_h,a^m_h)}(J^{\pi^m}_{h+1}) }}{\sqrt{n^{m-1}(s^m_h,a^m_h)}}.
    \end{align}
    We now bound each of the three sums in \cref{eq: rl final bound relation 2 UL}. 
    We bound the first sum in \cref{eq: rl final bound relation 2 UL} via standard analysis as follows:
    \begin{align*}
        \sum_{m=1}^M \sum_{h=1}^H & \frac{ H^3 \costbound^{-1} |S|L_m}{n^{m-1}( s^m_h,a^m_h)\vee 1}
        \le
        H^3 \costbound^{-1} |S|L_M \sum_{m=1}^M \sum_{h=1}^H\frac{1}{n^{m-1}( s^m_h,a^m_h)\vee 1}
        \\
        & =
        H^3 \costbound^{-1} |S|L_M \sum_{m=1}^M\sum_{s,a} \frac{ \sum_{h=1}^H\indevent{s^m_h=s,a^m_h=a}}{n^{m-1}(s,a)\vee 1}
        \\
        & \leq 
        H^3 \costbound^{-1} |S|L_M \sum_{m=1}^M\sum_{s,a} \indevent{n^{m-1}(s,a)\geq H} \frac{ \sum_{h=1}^H\indevent{s^m_h=s,a^m_h=a}}{n^{m-1}(s,a)\vee 1} + 2 H^4 \costbound^{-1} |S|^2 |A| L_M
        \\
        & \leq 
        3 H^3 \costbound^{-1} |S|^2 |A| L_M\log(MH) + 2 H^4 \costbound^{-1} |S|^2 |A|L_M,
    \end{align*}
    where the last inequality is by \cref{lemma: cumulative visitation bound stationary MDP} that bounds $\sum_{m,s,a} \indevent{n^{m-1}(s,a)\geq H} \frac{ \sum_{h=1}^H\indevent{s^m_h=s,a^m_h=a}}{n^{m-1}(s,a)\vee 1} \leq 3 |S| |A| \log(M H)$.
    
    The second sum in~\cref{eq: rl final bound relation 2 UL} is bounded as follows.
    \begin{align*}
        \sum_{m=1}^M\sum_{h=1}^H & \frac{\sqrt{ c(s^m_h,a^m_h)L_m}}{\sqrt{n^{m-1}(s^m_h,a^m_h)\vee 1}}
        \leq 
        \sum_{m=1}^M\sum_{h=1}^H \frac{\sqrt{ c(s^m_h,a^m_h)L_m}}{\sqrt{n^{m-1}(s^m_h,a^m_h)\vee 1}}\indevent{n^{m-1}(s^m_h,a^m_h)\geq H} + 2 H |S| |A| L_M
        \\
        & \overset{(a)}{\leq} 
        \sqrt{L_M}\sqrt{\sum_{m=1}^M\sum_{h=1}^H c(s^m_h,a^m_h)   } \cdot \sqrt{\sum_{m=1}^M\sum_{h=1}^H   \frac{\indevent{n^{m-1}(s^m_h,a^m_h)\geq H}}{n^{m-1}(s^m_h,a^m_h)\vee 1}} + 2 H |S| |A| L_M
        \\
        & \overset{(b)}{\leq}
        \sqrt{L_M}\sqrt{\sum_{m=1}^M\sum_{h=1}^H c(s^m_h,a^m_h)   } \cdot \sqrt{3 |S| |A| \log (MH)} + 2 H |S| |A| L_M
        \\
        & \le
        \sqrt{3 |S| |A|} L_M \sqrt{\sum_{m=1}^M\sum_{h=1}^H c(s^m_h,a^m_h) + c_f(s^m_{H+1})} + 2 H |S| |A| L_M
        \\
        & \le
        O \Bigl(  \sqrt{ \costbound |S||A| M } L_M + H^3 \costbound^{-1} |S|^2 |A| \log^{3/2} \frac{M H |S| |A|}{\delta} \Bigr).
    \end{align*}
    where (a) is by Cauchy-Schwartz, (b) is by \cref{lemma: cumulative visitation bound stationary MDP}, and the last inequality is by \cref{lemma: bound on cost term}.
    The third sum in \cref{eq: rl final bound relation 2 UL} is bounded in \cref{lemma: bound on variance term} by
    \begin{align*}
        & \sum_{m=1}^M \sum_{h=1}^H \frac{\sqrt{L_m\VAR_{P(\cdot \mid s^m_h,a^m_h)}(J^{\pi^m}_{h+1}) }}{\sqrt{n^{m-1}(s^m_h,a^m_h)}} 
        \leq 
        \sqrt{L_M}\sum_{m=1}^M \sum_{h=1}^H \frac{\sqrt{\VAR_{P(\cdot \mid s^m_h,a^m_h)}(J^{\pi^m}_{h+1}) }}{\sqrt{n^{m-1}(s^m_h,a^m_h)}} 
        \tag{$L_m$ increasing in $m$}
        \\
        & \qquad \qquad \leq  
        \sqrt{L_m} \cdot O \Bigl(  \sqrt{ \costbound^2 |S||A| M \log(MH)} + H^3 \costbound^{-1} |S|^2 |A| \log \frac{M H |S| |A|}{\delta} \Bigr). 
        \tag{Lemma~\ref{lemma: bound on variance term}}
    \end{align*}
\end{proof}

\subsection{Bounds on the cumulative bonuses}

\begin{lemma}[Bound on the Cumulative Cost Function Bonus]
    \label{lemma: bound on bonus bc UL RL}
    Conditioning on the good event the following bound holds for all $h\in [H]$.
    \[
        \sum_{m=1}^M b^m_c(s^m_h,a^m_h)
        \leq 
        \sum_{m=1}^M \sqrt{ \frac{2 c(s^m_h,a^m_h)L_m }{n^{m-1}(s^m_h,a^m_h)\vee 1}} + \sum_{m=1}^M \frac{10L_m}{ n^{m-1}(s^m_h,a^m_h)\vee 1}.
    \]
\end{lemma}

\begin{proof}
    By definition of $b^m_c$ and since the event $\cap_m E^{cv}(m)$ holds, we have
    \begin{align*}
        \sum_{m=1}^M b^m_c(s^m_h,a^m_h) 
        & = 
        \sum_{m=1}^M \sqrt{ \frac{2 \overline{\VAR}^{m-1}_{s^m_h,a^m_h}(c) L_m }{n^{m-1}(s^m_h,a^m_h)\vee 1}} +  \frac{5L_m}{ n^{m-1}(s^m_h,a^m_h)\vee 1}
        \\
        & \leq 
        \sum_{m=1}^M \sqrt{\frac{2 \VAR_{s^m_h,a^m_h}(c)L_m }{n^{m-1}(s^m_h,a^m_h)\vee 1}} + \sqrt{ \frac{2 L_m \, \abs1{\VAR_{s^m_h,a^m_h}(c)-\overline{\VAR}^{m-1}_{s^m_h,a^m_h,t-1}(c)} }{n^{m-1}(s^m_h,a^m_h)\vee 1}} +  \frac{5L_m}{ n^{m-1}(s^m_h,a^m_h)\vee 1}
        \\
        & \leq 
        \sum_{m=1}^M \sqrt{ \frac{2 \VAR_{s^m_h,a^m_h}(c)L_m }{n^{m-1}(s^m_h,a^m_h)\vee 1}} +  \frac{10L_m}{ n^{m-1}(s^m_h,a^m_h)\vee 1},
    \end{align*}
    where the first inequality holds since $\sqrt{a+b}\leq \sqrt{|a|}+\sqrt{|b|}$.
    Finally, notice that for every $(s,a) \in S \times A$ the variance of the cost is bounded by the second moment, which is bounded by the expected value $c(s,a)$ since the random cost value is bounded in $[0,1]$.
\end{proof}

\begin{lemma}[Bound on the Cumulative Transition Model Bonus]
    \label{lemma: bound on bonus bp UL RL}
    Conditioning on the good event the following bound holds for all $h\in [H]$.
    \begin{align*}
        \sum_{m=1}^M b^m_p(s^m_h,a^m_h)
        & \leq 
        \sum_{m=1}^M \frac{139H^3 \costbound^{-1} |S| L_m}{ n^{m-1}(s^m_h,a^m_h)\vee 1} + \sum_{m=1}^M \sqrt{2L_m}\frac{\sqrt{\VAR_{P(\cdot \mid s^m_h,a^m_h)}(J^{\pi^m}_{h+1})}}{\sqrt{n^{m-1}(s^m_h,a^m_h)\vee 1}}
        \\
        & \quad + 
        \frac{1}{8H}\sum_{m=1}^M \E_{P(\cdot \mid  s^m_h,a^m_h)} [ \bar{J}^m_{h+1}(s^m_{h+1}) - \underline{J}^m_{h+1}(s^m_{h+1})].
    \end{align*}
\end{lemma}

\begin{proof}
    First, by applying Lemma~\ref{lemma: transition different to next state expectation} with $\alpha=8H,C_1=C_2=2$ and $HL_m (2C_2+ \alpha |S| C_1/2)\leq 12 H^2 |S| L_m$, we have
    \begin{align}
        \nonumber
        \E_{\bar{P}^{m-1}(\cdot \mid s,a)}\brs{  \bar{J}^m_{h+1}(s') - \underline{J}^m_{h+1}(s')} 
        & = 
        \E_{P(\cdot \mid s,a)}\brs{ \bar{J}^m_{h+1}(s') - \underline{J}^m_{h+1}(s')} + (\bar{P}^{m-1}- P)(\cdot \mid s,a) \cdot ( \bar{J}^m_{h+1} - \underline{J}^m_{h+1} )
        \\
        \label{eq: useful relation bound on bonus bp}
        & \leq
        \frac{9}{8}\E_{P(\cdot \mid s,a)}\brs{ \bar{J}^m_{h+1}(s') - \underline{J}^m_{h+1}(s')} + \frac{ 12H^2|S| L_m}{n^{m-1}(s,a)\vee 1}.
    \end{align}
    Thus, the bonus $b^p_{t}(s,a)$ can be upper bounded as follows.
    \begin{align}
        \nonumber
         b^m_p(s,a)
        & \leq
        \sqrt{2}\sqrt{\frac{\VAR_{\bar{P}^{m-1}(\cdot \mid  s,a)}(\underline{J}^m_{h+1}) L_m}{n^{m-1}(s,a)\vee 1}}  + \frac{1}{16H}\E_{\bar{P}^{m-1}(\cdot \mid s,a)}\brs{ \bar{J}^m_{h+1}(s') - \underline{J}^m_{h+1}(s')} + \frac{ 62H^3 \costbound^{-1} |S| L_m}{n^{m-1}(s,a)\vee 1} 
        \\
        \label{eq:  bound on bp UL RL relation 1}
        & \leq  
        \sqrt{2}\sqrt{\frac{\VAR_{\bar{P}^{m-1}(\cdot \mid  s,a)}(\underline{J}^m_{h+1}) L_m}{n^{m-1}(s,a)\vee 1}}  + \frac{9}{128H}\E_{P(\cdot \mid s,a)}\brs{ \bar{J}^m_{h+1}(s') - \underline{J}^m_{h+1}(s')} + \frac{74 H^3 \costbound^{-1} |S| L_m}{n^{m-1}(s,a)\vee 1}.
    \end{align}
    We bound the first term of~\eqref{eq:  bound on bp UL RL relation 1} to establish the lemma. 
    It holds that
    \begin{align*}
        \sqrt{2L_m}\sqrt{\frac{\VAR_{\bar{P}^{m-1}(\cdot \mid  s,a)}(\underline{J}^m_{h+1}) }{n^{m-1}(s,a)\vee 1}}
        & =
        \underbrace{\sqrt{2L_m}\frac{\sqrt{\VAR_{\bar{P}^{m-1}(\cdot \mid  s,a)}(\underline{J}^m_{h+1})} -  \sqrt{\VAR_{P(\cdot \mid  s,a)}(J^*_{h+1})} }{\sqrt{n^{m-1}(s,a)\vee 1}} }_{(i)}
        \\
        & \quad +
        \underbrace{\sqrt{2L_m}\frac{ \sqrt{\VAR_{P(\cdot \mid  s,a)}(J^*_{h+1})} - \sqrt{\VAR_{P(\cdot \mid  s,a)}(J^{\pi^m}_{h+1})}}{\sqrt{n^{m-1}(s,a)\vee 1}}}_{(ii)} 
        \\
        & \quad +
        \frac{\sqrt{2L_m}\sqrt{\VAR_{P(\cdot \mid  s,a)}(J^{\pi^m}_{h+1})}}{\sqrt{n^{m-1}(s,a)\vee 1}}.
    \end{align*}
    Term $(i)$ is bounded by \cref{lemma: variance diff is upper bounded by value difference} (by setting $\alpha=32H$ and $(5+ \alpha/2)\costbound \leq 21H^2$),
    \[
       \sqrt{2L_m}\frac{\sqrt{\VAR_{\bar{P}^{m-1}(\cdot \mid  s,a)}(\underline{J}^m_{h+1})} -  \sqrt{\VAR_{P(\cdot \mid  s,a)}(J^*_{h+1})} }{\sqrt{n^{m-1}(s,a)\vee 1}} 
       \leq   
       \frac{1}{32H}\E_{\bar{P}^{m-1}(\cdot \mid s,a)}\brs*{J^*_{h+1}(s') - \underline{J}^m_{h+1}(s')} + \frac{21H^2L_m}{ n^{m-1}(s,a)\vee 1}.
    \]
    Following the same steps as in~\eqref{eq: useful relation bound on bonus bp}, we get
    \[
        \E_{\bar{P}^{m-1}(\cdot \mid s,a)}\brs*{J^*_{h+1}(s') - \underline{J}^m_{h+1}(s')} 
        \leq 
        \frac{9}{8}\E_{P(\cdot \mid s,a)}\brs*{J^*_{h+1}(s') - \underline{J}^m_{h+1}(s')} + \frac{12H^2|S| L_m}{n^{m-1}(s,a)\vee 1},
    \]
    and thus, 
    \[
        (i) 
        \leq 
        \frac{9}{256H}\E_{P(\cdot \mid s,a)}\brs*{J^*_{h+1}(s') - \underline{J}^m_{h+1}(s')} + \frac{33 H^2|S| L_m}{ n^{m-1}(s,a)\vee 1}.
    \]
    Term $(ii)$ is bounded as follows.
    \begin{align*}
        (ii)
        & \leq 
        \frac{ \sqrt{\VAR_{P(\cdot \mid  s,a)}(J^*_{h+1} - J^{\pi^m}_{h+1})}}{\sqrt{n^{m-1}(s,a)\vee 1}} 
        \tag{By Lemma~\ref{lemma: std difference}}
        \\
        & \leq 
        \frac{ \sqrt{\E_{P(\cdot \mid s,a)}[(J^*_{h+1}(s') - J^{\pi^m}_{h+1}(s'))^2]}}{\sqrt{n^{m-1}(s,a)\vee 1}}
        \\
        & \leq 
        \frac{ \sqrt{2H\E_{P(\cdot \mid s,a)}[(J^*_{h+1}(s') - J^{\pi^m}_{h+1}(s'))]}}{\sqrt{n^{m-1}(s,a)\vee 1}} 
        \tag{ $0 \leq J^*_h(s') - V^{\pi^m}_h(s')\leq 2H$ }
        \\
        & \leq 
        \frac{1}{64H}\E_{P(\cdot \mid s,a)}[( J^{\pi^m}_{h+1}(s') - J^*_{h+1}(s'))] + \frac{32H^2}{n^{m-1}(s,a)\vee 1}. 
        \tag{$ab\leq \frac{1}{\alpha}a^2 + \frac{\alpha}{4}b^2$ for $\alpha=64H$}
    \end{align*}
    Thus, applying $\bar{J}^m_h\ge J^{\pi^m}_h\ge J^*_h\ge \underline{J}^m_h$  (Lemma~\ref{lemma: optimism ucbvi-UL}) in the bounds of $(i)$ and $(ii)$ we get
    \[
         b^m_p(s,a)
        \leq  
        \frac{1}{8H}\E_{P(\cdot \mid s,a)}[(\bar{J}^m_h(s') - \underline{J}^m_h(s'))] + \frac{139 H^3 \costbound^{-1} |S| L_m}{ n^{m-1}(s,a)\vee 1} + \frac{\sqrt{2L_m}\sqrt{\VAR_{P(\cdot \mid  s,a)}(J^{\pi^m}_{h+1})}}{\sqrt{n^{m-1}(s,a)\vee 1}},
    \]
    and summing over $m$ concludes the proof.
\end{proof}

\begin{lemma}[Bound on Cost Term]
    \label{lemma: bound on cost term}
    Conditioning on the good event, it holds that
    \[
        \sum_{m=1}^M\sum_{h=1}^H c(s^m_h,a^m_h) + c_f(s^m_{H+1})
        \leq
        O \Bigl( \costbound M + H^5 \costbound^{-2} |S|^2 |A| \log \frac{M H |S| |A|}{\delta} \Bigr).
    \]
\end{lemma}

\begin{proof}
    Denote by $h_m$ the last time step before reaching an unknown state-action pair (or $H$ if it was not reached).
    By the event $E^{cost}$ we have
    \begin{align*}
        \sum_{m=1}^M\sum_{h=1}^H & c(s^m_h,a^m_h) + c_f(s^m_{H+1})
        =
        \sum_{m=1}^M \br*{ \sum_{h=h_m+1}^H c(s^m_h,a^m_h) + c_f(s^m_{h+1}) \indevent{h_m \ne H}}
        \\
        & \qquad +
        \sum_{m=1}^M \br*{ \sum_{h=1}^{h_m} c(s^m_h,a^m_h) + c_f(s^m_{h+1}) \indevent{h_m = H}}
        \\
        & \le
        2 \alpha H^5 \costbound^{-2} |S|^2 |A| \log \frac{M H |S| |A|}{\delta} + \sum_{m=1}^M \br*{ \sum_{h=1}^{h_m} c(s^m_h,a^m_h) + c_f(s^m_{h+1}) \indevent{h_m = H}}
        \\
        & \le
        10 \alpha H^5 \costbound^{-2} |S|^2 |A| \log \frac{M H |S| |A|}{\delta} + 2 \sum_{m=1}^M \E\brs*{ \sum_{h=1}^{h_m} c(s^m_h,a^m_h) + c_f(s^m_{h+1}) \indevent{h_m = H} ~\bigg\vert~ \trajconcat{m}}
        \\
        & \le
        O \Bigl( H^5 \costbound^{-2} |S|^2 |A| \log \frac{M H |S| |A|}{\delta} + \costbound M \Bigr),
    \end{align*}
    where the second inequality follows since every state-action pair becomes known after the number of visits is $\alpha H^4 \costbound^{-2} |S| \log \frac{M H |S| |A|}{\delta}$, and the last one by \cref{lem:second-moment-bound-finite-horizon}.
\end{proof}

\begin{lemma}[Bound on Variance Term]
    \label{lemma: bound on variance term}
    Conditioning on the good event, it holds that
    \[
        \sum_{m=1}^M\sum_{h=1}^H \frac{\sqrt{\VAR_{P(\cdot \mid s^m_h,a^m_h)}(J^{\pi^m}_{h+1}) }}{\sqrt{n^{m-1}(s^m_h,a^m_h)}}
        \leq
        O \Bigl( \sqrt{ \costbound^2 |S||A| M \log(MH)} + H^3 \costbound^{-1} |S|^{3/2} |A| \log \frac{M H |S| |A|}{\delta} \Bigr).
    \]
\end{lemma}

\begin{proof}
    Applying Cauchy-Schwartz inequality we get
    \begin{align*}
        \sum_{m=1}^M\sum_{h=1}^H & \frac{\sqrt{\VAR_{P(\cdot \mid s^m_h,a^m_h)}(J^{\pi^m}_{h+1}) }}{\sqrt{n^{m-1}(s^m_h,a^m_h) \vee 1}}
        \leq 
        \sum_{m=1}^M\sum_{h=1}^H \frac{\sqrt{\VAR_{P(\cdot \mid s^m_h,a^m_h)}(J^{\pi^m}_{h+1}) }}{\sqrt{n^{m-1}(s^m_h,a^m_h) \vee 1}}\indevent{n^{m-1}(s^m_h,a^m_h)\geq H} + 2 H^2 |S| |A|
        \\
        & \leq
        \sqrt{\sum_{m=1}^M\sum_{h=1}^H  \VAR_{P(\cdot \mid s^m_h,a^m_h)}(J^{\pi^m}_{h+1})} \sqrt{\sum_{m=1}^M\sum_{h=1}^H \frac{1}{n^{m-1}(s^m_h,a^m_h)\vee 1}\indevent{n^{m-1}(s^m_h,a^m_h)\geq H}} + 2 H^2 |S| |A|
        \\
        & \leq 
        \sqrt{\sum_{m=1}^M\sum_{h=1}^H  \VAR_{P(\cdot \mid s^m_h,a^m_h)}(J^{\pi^m}_{h+1})} \sqrt{3 |S| |A| \log(MH)} + 2 H^2 |S| |A| 
        \tag{Lemma~\ref{lemma: cumulative visitation bound stationary MDP}}
        \\
        & \leq
        \sqrt{2 \sum_{m=1}^M \E\brs*{\sum_{h=1}^H  \VAR_{P(\cdot \mid s^m_h,a^m_h)}(J^{\pi^m}_{h+1}) \mid  \trajconcat{m} } + 16 H^3 L_M}\sqrt{3 |S||A| \log(MH)}  + 2 H^2 |S| |A|  
        \tag{Event $E^{\VAR}$ holds} 
        \\
        & \leq 
        3 \sqrt{\sum_{m=1}^M \E\brs*{\sum_{h=1}^H  \VAR_{P(\cdot \mid s^m_h,a^m_h)}(J^{\pi^m}_{h+1}) \mid \trajconcat{m}}} \sqrt{|S||A| \log(MH)} 
        \\
        & \quad +
        7 \sqrt{|S||A| H^3 \log(MH) L_M} + 2 H^2 |S| |A|
        \tag{$\sqrt{a+b}\leq \sqrt{a} + \sqrt{b}$}
        \\
        & \overset{(a)}{=} 
        3 \sqrt{\sum_{m=1}^M \E\brs*{\br*{ \sum_{h=1}^H c(s^m_h,a^m_h) + c_f(s^m_{h+1}) - J_1^{\pi^m}(s_1)}^2 ~\bigg\vert~ \trajconcat{m}}} \sqrt{|S||A| \log(MH)}
        \\
        & \quad +
        7 \sqrt{|S| |A| H^3 \log(MH)) L_m} + 2 H^2 |S| |A|
        \\
        & \overset{(b)}{\leq} 
        3 \sqrt{\sum_{m=1}^M \E\brs*{\br*{ \sum_{h=1}^H c(s^m_h,a^m_h) + c_f(s^m_{h+1})}^2 ~\bigg\vert~ \trajconcat{m}}} \sqrt{|S||A| \log(MH)} + 9 H^2 |S| |A| L_M
        \\
        & \leq 
        O \Bigl( \sqrt{ \costbound^2 |S||A| M \log(MH)} + H^3 \costbound^{-1} |S|^{3/2} |A| \log \frac{M H |S| |A|}{\delta} \Bigr),
    \end{align*}
    where (a) is by law of total variance~\cite{azar2017minimax}, see Lemma~\ref{lemma: law of total variance for RL}, (b) is because the variance is bounded by the second moment, and the last inequality is by \cref{lem:finite-horizon-second-moment}.
\end{proof}

\subsection{Bounds on the second moment}

\begin{lemma}
    \label{lem:finite-horizon-second-moment}
    Conditioning on the good event, it holds that
    \[
        \sum_{m=1}^M \E\brs*{\br*{ \sum_{h=1}^H c(s^m_h,a^m_h) + c_f(s^m_{h+1})}^2 ~\bigg\vert~ \trajconcat{m}}
        \le
        O \Bigl( \costbound^2 M + H^6 \costbound^{-2} |S|^2 |A| \log \frac{M H |S| |A|}{\delta} \Bigr).
    \]
\end{lemma}

\begin{proof}
    Denote by $h_m$ the last time step before reaching an unknown state-action pair (or $H$ if it was not reached).
    By the event $E^{Sec1}$ we have
    \begin{align*}
        \sum_{m=1}^M & \E\brs*{\br*{ \sum_{h=1}^H c(s^m_h,a^m_h) + c_f(s^m_{h+1})}^2 ~\bigg\vert~ \trajconcat{m}}
        \le
        2 \sum_{m=1}^M \br*{ \sum_{h=1}^H c(s^m_h,a^m_h) + c_f(s^m_{h+1})}^2 + 62 H^4 L_M
        \\
        & \le
        4 \sum_{m=1}^M \br*{ \sum_{h=h_m+1}^H c(s^m_h,a^m_h) + c_f(s^m_{h+1}) \indevent{h_m \ne H}}^2 + 62 H^4 L_M
        \\
        & \qquad +
        4 \sum_{m=1}^M \br*{ \sum_{h=1}^{h_m} c(s^m_h,a^m_h) + c_f(s^m_{h+1}) \indevent{h_m = H}}^2
        \\
        & \le
        300 \alpha H^6 \costbound^{-2} |S|^2 |A| \log \frac{M H |S| |A|}{\delta} + 4 \sum_{m=1}^M \br*{ \sum_{h=1}^{h_m} c(s^m_h,a^m_h) + c_f(s^m_{h+1}) \indevent{h_m = H}}^2
        \\
        & \le
        400 \alpha H^6 \costbound^{-2} |S|^2 |A| \log \frac{M H |S| |A|}{\delta} + 4 \sum_{m=1}^M \E\brs*{\br*{ \sum_{h=1}^{h_m} c(s^m_h,a^m_h) + c_f(s^m_{h+1}) \indevent{h_m = H}}^2 ~\bigg\vert~ \trajconcat{m}}
        \\
        & \le
        O \Bigl( H^6 \costbound^{-2} |S|^2 |A| \log \frac{M H |S| |A|}{\delta} + \costbound^2 M \Bigr),
    \end{align*}
    where the third inequality follows since every state-action pair becomes known after the number of visits is $\alpha H^4 \costbound^{-2} |S| \log \frac{M H |S| |A|}{\delta}$, the forth inequality by event $E^{Sec2}$, and the last one by \cref{lem:second-moment-bound-finite-horizon}.
\end{proof}

\begin{lemma}
    \label{lem:second-moment-bound-finite-horizon}
    Let $m$ be an episode and $h_m$ be the last time step before an unknown state-action pair was reached (or $H$ if they were not reached).
    Further, denote by $C^m = \sum_{h=1}^{h_m} c(s^m_h,a^m_h) + c_f(s^m_{H+1}) \indevent{h_m = H}$ the cumulative cost in the episode until time $h_m$.
    Then, under the good event, $\bbE\brk[s]{C^m \mid \trajconcat{m}} \le 3 \costbound$ and $\bbE\brk[s]{(C^m)^2 \mid \trajconcat{m}} \le 2 \cdot 10^4 \costbound^2$.
\end{lemma}

\begin{proof}
    Consider the following finite-horizon MDP $\calM^m = (S \cup \{ \ssink \},A,P^m,H,c^m,c_f^m)$ that contracts unknown state-action pairs with a new goal state, i.e., $c^m(s,a) = c(s,a) \indevent{s \ne \ssink}$ and $c^m_f(s) = c_f(s) \indevent{s \ne \ssink}$ and 
    \[
        P^m_h(s' \mid s,a)
        =
        \begin{cases}
            0, \quad & (s',\pi^m_{h+1}(s')) \text{ is unknown};
            \\
            P(s' \mid s,a), \quad & s' \ne \ssink \text{ and } (s',\pi^m_{h+1}(s')) \text{ is known};
            \\
            1 - \sum_{s'' \in S} P^m_h(s'' \mid s,a), & s' = \ssink.
        \end{cases}
    \]
    Denote by $J^m$ the cost-to-go function of $\pi^m$ in the MDP $\calM^m$.
    Moreover, we slightly abuse notation to let $\wt P^m$ be the transition function induced by $\bar P^{m-1}$ in the MDP $\calM^m$ similarly to $P^m$, and $\wt J^m$ the cost-to-go function of $\pi^m$ with respect to $\bar P^{m-1}$ (and cost function $\tilde c^m = \bar c^{m-1} - b^m_c - b^m_p$).
    By the value difference lemma (see, e.g., \citealp{efroni2020optimistic}), for every $s,h$ such that $(s,\pi^m_h(s))$ is known,
    \begin{align*}
        & J^m_h(s)
        =
        \wt J^m_h(s) + \sum_{h'=h}^H \bbE \Bigl[ c^m(s_{h'},a_{h'}) - \tilde c^m_{h'}(s_{h'},a_{h'}) \mid s_h=s,P^m,\pi^m \Bigr] 
        \\
        & \qquad + 
        \sum_{h'=h}^H \bbE \Bigl[ \bigl( P^m_{h'}(\cdot \mid s_{h'},a_{h'}) - \wt P^m_{h'}(\cdot \mid s_{h'},a_{h'}) \bigr) \cdot \wt J^m \mid s_h=s,P^m,\pi^m \Bigr]
        \\
        & \le
        \wt J^m_h(s) + H \max_{\substack{(s,\pi^m_{h'}(s)) \\ \text{known}}} | c(s,\pi^m_{h'}(s)) - \tilde c^m_{h'}(s,\pi^m_{h'}(s)) | 
        + 
        H \lVert \wt J^m \rVert_\infty \max_{\substack{(s,\pi^m_{h'}(s)) \\ \text{known}}} \lVert P^m_{h'}(\cdot | s,\pi^m_{h'}(s)) - \wt P^m_{h'}(\cdot | s,\pi^m_{h'}(s))  \rVert_1
        \\
        & \le
        \wt J^m_h(s) + H \max_{\substack{(s,\pi^m_{h'}(s)) \\ \text{known}}} | c(s,\pi^m_{h'}(s)) - \tilde c^m_{h'}(s,\pi^m_{h'}(s)) | 
        \\
        & \qquad + 
        2 H \lVert \wt J^m \rVert_\infty \max_{\substack{(s,\pi^m_{h'}(s)) \\ \text{known}}} \lVert P(\cdot | s,\pi^m_{h'}(s)) - \bar P^{m-1}(\cdot | s,\pi^m_{h'}(s)) \rVert_1
        \\
        & \le
        J^*_h(s) + H \max_{\substack{(s,\pi^m_{h'}(s)) \\ \text{known}}} | c(s,\pi^m_{h'}(s)) - \tilde c^m_{h'}(s,\pi^m_{h'}(s)) | 
        + 
        2 H \costbound \max_{\substack{(s,\pi^m_{h'}(s)) \\ \text{known}}} \lVert P(\cdot | s,\pi^m_{h'}(s)) - \bar P^{m-1}(\cdot | s,\pi^m_{h'}(s)) \rVert_1,
    \end{align*}
    where the last inequality follows by optimism and since $J^\star_h(s) \le \costbound$. 
    Thus, by \cref{sec:ULCVI-admissible} (since all state-action pairs in $\calM^m$ are known), we have that $J^m_h(s) \le J^*_h(s) + 2 \costbound \le 3 \costbound$.
    Notice that $C^m$ is exactly the cost in the MDP $\calM^m$, so $\bbE\brk[s]{C^m \mid \trajconcat{m}} \le 3 \costbound$.
    
    Similarly, we notice that $\bbE \brk[s]{(C^m)^2 \mid \trajconcat{m}} = \bbE \brk[s]{(\wh C)^2}$, where $\wh C$ is the cumulative cost in $\calM^m$, and we override notation by denoting $\wh C = \sum_{h=1}^H c(s_h,a_h) + c_f(s_{H+1})$.
    We split the time steps into $Q$ blocks as follows.
    We denote by $t_1$ the first time step in which we accumulated a total cost of at least $3 \costbound$ (or $H+1$ if it did not occur), by $t_2$ the first time step in which we accumulated a total cost of at least $3 \costbound$ after $t_1$, and so on up until $t_Q = H+1$. 
    Then, the first block consists of time steps $t_0=1,\dots,t_1-1$, the second block consists of time steps $t_1,\dots,t_2-1$, and so on.
    Since $J_h^m(s) \le 3 \costbound$ we must have $c(s_h,a_h) \le 3 \costbound$ for all $h=1,\ldots,H$ and thus in every such block the total cost is between $3 \costbound$ and $6 \costbound$. Thus, 
    \begin{align*}
        \bbE \brk[s]4{\brk4{\sum_{h=1}^{H} c(s_h,a_h) + c_f(s_{H+1})}^2}  
        & \ge
        \bbE \brk[s]4{\sum_{h=1}^{H} c(s_h,a_h) + c_f(s_{H+1})}^2
        \\
        & =
        \bbE \brk[s]4{\sum_{i=0}^{Q-1} \sum_{h=t_i}^{t_{i+1}-1} c(s_h,a_h) + c_f(s_{H+1})}^2
        \\
        & \ge
        \bbE [3 \costbound Q]^2
        =
        9 \costbound^2 \bbE[Q]^2,
    \end{align*}
    by Jensen's inequality.
    On the other hand,
    \begin{align*}
        \bbE & \brk[s]4{\brk4{\sum_{h=1}^{H} c(s_h,a_h) + c_f(s_{H+1})}^2}
        =
        \bbE \brk[s]4{\brk4{\sum_{h=1}^{H} c(s_h,a_h) + c_f(s_{H+1}) - J_1^m(s_1) + J_1^m(s_1) }^2}
        \\
        & \qquad \le
        2 \bbE \brk[s]4{\brk4{\sum_{h=1}^{H} c(s_h,a_h) + c_f(s_{H+1}) - J_1^m(s_1)}^2}
        + 
        2 J_1^m(s_1)^2 \\
        & \qquad \le
        2 \bbE \brk[s]4{ \brk4{ \sum_{i=0}^{Q-1} \sum_{h=t_i}^{t_{i+1}-1} c(s_h,a_h)  - J^m_{t_i}(s_{t_i}) + J^m_{t_{i+1}}(s_{t_{i+1}})}^2}
        +
        18 \costbound^2 \\
        & \qquad \overset{(a)}{=}
        4 \bbE \brk[s]4{\sum_{i=0}^{Q-1} \brk4{\sum_{h=t_i}^{t_{i+1}-1} c(s_h,a_h) - J^m_{t_i}(s_{t_i}) + J^m_{t_{i+1}}(s_{t_{i+1}})}^2}
        +
        18 \costbound^2 \\
        & \qquad \le
        4 \bbE [Q \cdot (9 \costbound)^2] + 18 \costbound^2
        \le
        324 \costbound^2 \bbE [Q] + 18 \costbound^2.
    \end{align*}
    For (a) we used the fact that $\bbE\brk[s]{\sum_{h=t_i}^{t_{i+1}-1} c(s_h,a_h) - J_{t_i}(s_{t_i}) + J_{t_{i+1}}(s_{t_{i+1}})} = 0$ using the Bellman optimality equations and conditioned on all past randomness up until time $t_i$, and the fact that $t_{i+1}$ is a stopping time, in the following manner,
    \begin{align*}
        \bbE\brk[s]4{\sum_{h=t_i}^{t_{i+1}-1} c(s_h,a_h) - J^m_{t_i}(s_{t_i}) + J^m_{t_{i+1}}(s_{t_{i+1}})}
        & =
        \bbE\brk[s]4{\sum_{h=t_i}^{t_{i+1}-1} c(s_h,a_h) - J^m_h(s_h) + J^m_{h+1}(s_{h+1})}
        \\
        & =
        \bbE\brk[s]4{\sum_{h=t_i}^{t_{i+1}-1} \bbE\brk[s]2{c(s_h,a_h) - J^m_h(s_h) + J^m_{h+1}(s_{h+1}) \mid s_h}}
        \\
        & =
        \bbE\brk[s]4{\sum_{h=t_i}^{t_{i+1}-1} c(s_h,a_h) + \bbE\brk[s]2{ J^m_{h+1}(s_{h+1}) \mid s_h} - J^m_h(s_h)}
        =
        0.
    \end{align*}
    Thus, we have
    \[
        9 \costbound^2 \bbE[Q]^2
        \le
        324 \costbound^2 \bbE [Q] + 18 \costbound^2,
    \]
    and solving for $\bbE[Q]$ we obtain $\bbE[Q] \le 37$, so
    \[
        \bbE\brk[s]{(C^m)^2 \mid \trajconcat{m}}
        =
        \bbE \brk[s]4{\brk4{\sum_{h=1}^{H} \hat c(s_h,a_h) + \hat c_f(s_{H+1})}^2} 
        \le 
        2 \cdot 10^4 \costbound^2.
    \]
\end{proof}

\begin{lemma}[Variance Difference is Upper Bounded by Value Difference]
    \label{lemma: variance diff is upper bounded by value difference}
    Assume that the value at time step $h+1$ is optimistic, i.e., $\underline{J}^m_{h+1}(s) \le J^*_{h+1}(s)$ for all $s\in S$.
    Conditioning on the event $\cap_m E^{pv2}(m)$ it holds for all $(s,a) \in S \times A$ that
    \[
        \sqrt{2L_m} \frac{ \abs*{\sqrt{\VAR_{\bar{P}^{m-1}(\cdot \mid  s,a)}(\underline{J}^m_{h+1})} - \sqrt{\VAR_{P(\cdot \mid  s,a)}(J^*_{h+1})}}}{\sqrt{n^{m-1}(s,a) \vee 1}}
        \le  \frac{1}{\alpha} \E_{\bar{P}^{m-1}(\cdot \mid s,a)}\brs*{J^*_{h+1}(s') - \underline{J}^m_{h+1}(s')} + \frac{(5+ \alpha/2) \costbound L_m}{n^{m-1}(s,a) \vee 1},
    \]
    for any $\alpha > 0$.
\end{lemma}

\begin{proof}
    Conditioning on $\cap_m E^{pv2}(m)$, the following relations hold.
    \begin{align*}
        \abs*{\sqrt{\VAR_{\bar{P}^{m-1}(\cdot \mid  s,a)}(\underline{J}^m_{h+1})} - \sqrt{\VAR_{P(\cdot \mid  s,a)}(J^*_{h+1})}}
        & \leq \abs*{\sqrt{\VAR_{\bar{P}^{m-1}(\cdot \mid  s,a)}(\underline{J}^m_{h+1})} - \sqrt{\VAR_{\bar{P}^{m-1}(\cdot \mid  s,a)}(J^*_{h+1})}} 
        \\
        & \qquad +
        \sqrt{\frac{12\costbound^2 L_m}{n^{m-1}(s,a) \vee 1}}
        \\
        & \le
        \sqrt{\VAR_{\bar{P}^{m-1}(\cdot \mid  s,a)}(J^*_{h+1} - \underline{J}^m_{h+1})}  + \sqrt{\frac{12\costbound^2 L_m}{n^{m-1}(s,a) \vee 1}} 
        \\
        &\leq \sqrt{\E_{\bar{P}^{m-1}}\brs*{(J^*_{h+1}(s') - \underline{J}^m_{h+1}(s'))^2}}  + \sqrt{\frac{12\costbound^2 L_m}{n^{m-1}(s,a) \vee 1}} \\
        &\leq \sqrt{\costbound \E_{\bar{P}^{m-1}}\brs*{ J^*_{h+1}(s') - \underline{J}^m_{h+1}(s')}}  + \sqrt{\frac{12\costbound^2 L_m}{n^{m-1}(s,a) \vee 1}},
    \end{align*}
    where the second inequality is by \cref{lemma: std difference}, and the last relation holds since $J^*_{h+1}(s'),\underline{J}^m_{h+1}(s')\in [0,\costbound]$ (the first, by model assumption, and the second, by the update rule) and since $J^*_{h+1}(s')\geq \underline{J}^m_{h+1}(s')$ by the assumption the value is optimistic. 
    Thus,
    \begin{align*}
        \sqrt{2L_m} \frac{ \abs*{\sqrt{\VAR_{\bar{P}^{m-1}(\cdot \mid  s,a)}(\underline{J}^m_{h+1})} - \sqrt{\VAR_{P(\cdot \mid  s,a)}(J^*_{h+1})}}}{\sqrt{n^{m-1}(s,a)}}
        & \leq \sqrt{\E_{\bar{P}^{m-1}}\brs*{J^*_{h+1}(s') - \underline{J}^m_{h+1}(s')}} \sqrt{\frac{2 \costbound L_m}{n^{m-1}(s,a)\vee 1}}  
        \\
        & \qquad + 
        \frac{\sqrt{24} \costbound L_m}{n^{m-1}(s,a) \vee 1}
        \\
        & \leq
        \frac{1}{\alpha} \E_{\bar{P}^{m-1}}\brs*{J^*_{h+1}(s') - \underline{J}^m_{h+1}(s')}   + \frac{ (5+\alpha/2)\costbound L_m}{n^{m-1}(s,a) \vee 1},
    \end{align*}
    where the last inequality is by Young's inequality, $ab\leq \frac{1}{\alpha}a^2 + \frac{\alpha}{4}b^2$.
\end{proof}

\subsection{Useful results for reinforcement learning analysis}

\begin{lemma}[Cumulative Visitation Bound for Stationary MDP, e.g., \citealp{efroni2020reinforcement}, Lemma 23] 
    \label{lemma: cumulative visitation bound stationary MDP}
    It holds that
    \[
        \sum_{m=1}^M\sum_{s,a} \indevent{n^{m-1}(s,a)\geq H} \frac{ \sum_{h=1}^H\indevent{s^m_h=s,a^m_h=a}}{n^{m-1}(s,a)\vee 1}
        \leq 
        3 |S| |A| \log(M H).
    \]
\end{lemma}

\begin{proof}
    Recall that we recompute the optimistic policy only in the end of episodes in which the number of visits to some state-action pair was doubled.
    In this proof we refer to a sequence of consecutive episodes in which we did not perform a recomputation of the optimistic policy by the name of \emph{epoch}.
    Let $E$ be the number of epochs and note that $E \le |S| |A| \log (MH)$ because the number of visits to each state-action pair $(s,a)$ can be doubled at most $\log (MH)$ times.
    Next, denote by $\tilde n^e(s,a)$ the number of visits to $(s,a)$ until the end of epoch $e$ and by $\wt N^e(s,a)$ the number of visits to $(s,a)$ during epoch $e$.
    The following relations hold for any fixed $(s,a)$ pair.
    \begin{align*}
        \sum_{m=1}^M \indevent{n^{m-1}(s,a)\geq H} & \frac{ \sum_{h=1}^H\indevent{s^m_h=s,a^m_h=a}}{n^{m-1}(s,a)\vee 1} =
        \\
        & =
        \sum_{e=1}^E \indevent{\tilde n^{e-1}(s,a)\geq H} \frac{\wt N^e(s,a)}{\tilde n^{e-1}(s,a)}
        \\
        & =
        \sum_{e=1}^E \indevent{\tilde n^{e-1}(s,a)\geq H} \frac{\wt N^e(s,a)}{\tilde n^e(s,a)} \frac{\tilde n^e(s,a)}{\tilde n^{e-1}(s,a)} 
        \\
        & \leq 
        3 \sum_{e=1}^E \indevent{\tilde n^{e-1}(s,a)\geq H} \frac{\wt N^e(s,a)}{\tilde n^e(s,a)} 
        \\
        & = 
        3 \sum_{e=1}^E \indevent{\tilde n^{e-1}(s,a)\geq H}\frac{ \tilde n^e(s,a) - \tilde n^{e-1}(s,a)}{n^e(s,a)}
        \\
        & \leq 
        3 \sum_{e=1}^E \indevent{\tilde n^{e-1}(s,a)\geq H} \log\br*{\frac{\tilde n^e(s,a)}{\tilde n^{e-1}(s,a)}}
        \\
        & \leq 
        3 \indevent{\tilde n^E(s,a) \ge H}(\log \tilde n^E(s,a) - \log(H))
        \\
        & \leq 
        3 \log \br*{\tilde n^E(s,a)\vee 1},
    \end{align*}
    where the first inequality follows since $\frac{\tilde n^e(s,a)}{\tilde n^{e-1}(s,a)}\leq \frac{2 \tilde n^{e-1}(s,a)+H}{\tilde n^{e-1}(s,a)}\leq 3$ for $\tilde n^{e-1}(s,a)\ge H$, and the second inequality follows by the inequality $\frac{a-b}{a}\leq \log\frac{a}{b}$ for $a\geq b>0$.
    Applying Jensen's inequality we conclude the proof:
    \begin{align*}
        \sum_{m=1}^M \sum_{s,a} \indevent{n^{m-1}(s,a)\geq H} \frac{ \sum_{h=1}^H\indevent{s^m_h=s,a^m_h=a}}{n^{m-1}(s,a)\vee 1}
        & \leq 
        3 \sum_{s,a} \log \br*{\tilde n^E(s,a)\vee 1}
        \\
        & \leq 
        3 |S| |A| \log \br*{\sum_{s,a} \tilde n^E(s,a)}
        \\
        & \leq  
        3 |S| |A| \log \br*{MH}.
    \end{align*}
\end{proof}

\begin{lemma}[Transition Difference to Next State Expectation,~\citealp{efroni2021confidence}, Lemma 28]
    \label{lemma: transition different to next state expectation}
    Let $Y \in \mathbb{R}^{|S|}$ be a vector such that $0 \leq Y(s) \leq 2H$ for all $s\in S$. 
    Let $P_1$ and $P_2$ be two transition models and $n \in \mathbb{R}^{|S||A|}_+$.
    Let $\Delta P (\cdot \mid  s,a) \in \mathbb{R}^{|S|}$ and $\Delta P (s'| s,a)=  P_{1} (s'| s,a) -  P_{2} (s'| s,a)$. 
    Assume that  
    \[
        \forall (s,a,s') \in S \times A \times S, h \in [H] :\ 
        |\Delta P (s'| s,a)| 
        \le 
        \sqrt{\frac{ C_1 L_m P_{1}(s'|s,a) }{n(s,a) \vee 1}} + \frac{C_2 L_m}{n(s,a)\vee 1},
    \]
    for some $C_1 , C_2 > 0$.
    Then, for any $\alpha>0$.
    \[
        \abs*{\Delta P (\cdot \mid  s,a) \cdot Y}
        \leq 
        \frac{1}{\alpha} \E_{P_1(\cdot \mid s,a)}\brs*{ Y(s')} + \frac{H L_m(2 C_2+ \alpha |S| C_1/2)}{n(s,a)\vee 1}.
    \]
\end{lemma}

\begin{lemma}[Law of Total Variance, e.g.,~\citealp{azar2017minimax}]
    \label{lemma: law of total variance for RL}
    For any $\pi$ the following holds.
    \[
        \E\brs*{\sum_{h=1}^H\VAR_{P(\cdot |s_{h},a_h)}(J^\pi_{h+1} ) \mid \pi} 
        = 
        \E\brs*{\br*{\sum_{h=1}^H c(s_h,a_h) + c_f(s_{H+1}) - J_1^\pi(s_1) }^2 \mid \pi}.
    \]
\end{lemma}

\newpage

\section{Extending the reduction to unknown \texorpdfstring{$\costbound$}{}}
\label{sec:unknown-B-appendix}

In this section we assume $\costbound \ge 1$ to simplify presentation, but the results work similarly for $\costbound < 1$.
To handle unknown $\costbound$, we leverage techniques from the adversarial SSP literature \citep{rosenberg2020stochastic,chen2021finding} for learning the diameter of an SSP problem.
Recall that the SSP-diameter $D$ \citep{tarbouriech2019noregret} is defined as $D = \max_{s \in S} \min_{\pi:s \to A} \policytime{\pi}(s)$.
So to compute $D$ we can find the optimal policy with respect to the constant cost function $c_1(s,a) = 1$, and compute its cost-to-go function.
\citet{rosenberg2020stochastic} utilize this observation to estimate the SSP-diameter.
They show that one can estimate the expected time from a state $s$ to the goal state $\ssink$ by running the \verb|Bernstein-SSP| algorithm of \citet{rosenberg2020near} with unit costs for $L = \wt O(D^2 |S|^2 |A|)$ episodes and setting the estimator to be the average cost per episode times $10$.

Inspired by their approach, we use the \verb|Bernstein-SSP| algorithm on the the actual costs, in order to estimate the expected cost of the optimal policy.
Although \verb|Bernstein-SSP| suffers from sub-optimal regret, we run it only for a small number of episodes and therefore we will only suffer from a slightly larger additive factors in our regret bound, but keep minimax optimal regret for large enough $K$.

By similar proofs to Lemmas 26 and 27 from \citet[Appendix J]{rosenberg2020stochastic}, we can show that the cost-to-go from state $s$ can be estimated up to a constant multiplicative factor by running \verb|Bernstein-SSP| for $L = \wt O(\timebound^2 |S|^2 |A|)$ episodes.
This is demonstrated in the following lemma, where the upper bound follows from the regret guarantees of \verb|Bernstein-SSP| and the lower bound follows from concentration arguments (and noticing that the regret is minimized by playing the optimal policy, but even then it is not zero).

\begin{lemma}
    \label{lem:estimate-B-with-Bernstein-SSP}
    Let $s \in S$ and $L \ge 2400 \timebound^2 |S|^2 |A| \log^3 \frac{K \timebound |S| |A|}{\delta}$.
    Run \verb|Bernstein-SSP| with initial state $s$ for $L$ episodes and denote by $\wt B_s$ the average cost per episode times $10$.
    Then, with probability $1 - \delta$,
    \[
        \ctg{\piopt}(s)
        \le
        \wt B_s
        \le
        O(\costbound).
    \]
\end{lemma}

Thus, we use the first $L$ visits to each state in order to estimate its cost-to-go.
A state which was visited at least $L$ times will be called \emph{$\costbound$-known}, and otherwise \emph{$\costbound$-unknown} (not to be confused with our previous definition of known state-action pair).
To that end, we split the total time steps into $E$ epochs.
In epoch $e$, we apply our reduction to a virtual MDP $\calM^e$ that is identical to $\calM$ in $\costbound$-known states, but turns $\costbound$-unknown states into zero-cost sinks (like the goal state).
For every state $s \in S$ we maintain a \verb|Bernstein-SSP| algorithm $\calB_s$.
Every time we reach a $\costbound$-unknown state $s$, we run an episode of $\calB_s$ until the goal is reached.

Note that in the virtual MDP $\calM^e$ we can compute an upper bound on the optimal cost-to-go using our estimates.
Epoch $e$ ends once some $\costbound$-unknown state $s$ is visited $L$ times and thus becomes $\costbound$-known.
Therefore the number of epochs $E$ is bounded by $|S|$.
The important change, introduced by \citet{chen2021finding}, is  to not completely initialize our finite-horizon algorithm $\calA$ in the beginning of a new epoch as this leads to an extra $|S|$ factor in the regret.
Instead, algorithm $\calA$ inherits the experience (i.e., visit counters and accumulated costs) of the previous epoch in $\costbound$-known states.

The reduction without knowledge of $\costbound$ is presented in \cref{alg:ssp-reduction-unknown-B}, and next we prove that it maintains the same regret bound up to a slightly larger additive factor.

\begin{theorem}
    \label{thm:regret-bound-with-admissible-algorithm-unknown-B}
    Let $\calA$ be an admissible algorithm for regret minimization in finite-horizon MDPs and denote its regret in $M$ episodes by $\wh \calR_\calA (M)$.
    Then, running \cref{alg:ssp-reduction-unknown-B} with $\calA$ ensures that, with probability at least $1 - 2 \delta$, 
    \begin{align*}
        \regret 
        & \le
        \wh \calR_\calA \brk*{4 K + 4\cdot 10^4 |S| |A| \knownthresh \log \frac{K \timebound |S| |A| \knownthresh}{\delta} + 4\cdot 10^4 \timebound^2 |S|^3 |A| \log^3 \frac{K \timebound |S| |A|}{\delta}}
        \\
        & \qquad + 
        O \brk*{\costbound \sqrt{K \log \frac{K \timebound |S| |A| \knownthresh}{\delta}} + \timebound \knownthresh |S| |A| \log^2 \frac{K \timebound |S| |A| \knownthresh}{\delta} + \timebound^3 |S|^3 |A| \log^4 \frac{K \timebound |S| |A|}{\delta}},
    \end{align*}
    where $\knownthresh$ is a quantity that depends on the algorithm $\calA$ and on $|S|,|A|,H$.
\end{theorem}

Using the reduction with the \verb|ULCVI| algorithm, we can again obtain optimal regret for SSP.

\begin{theorem}
    \label{thm:optimal-regret-bound-unknown-B} 
    Running the reduction in \cref{alg:ssp-reduction-unknown-B} with the finite-horizon regret minimization algorithm \verb|ULCVI| ensures, with probability at least $1 - 2 \delta$, 
    \[
        \regret 
        = 
        O \brk3{\costbound \sqrt{|S| |A| K} \log \frac{K \timebound |S| |A|}{\delta} + \timebound^5 |S|^2 |A| \log^6 \frac{K \timebound |S| |A|}{\delta} + \timebound^3 |S|^3 |A| \log^4 \frac{K \timebound |S| |A|}{\delta}}.
    \]
\end{theorem}

\begin{algorithm}[H]
    \caption{\sc Reduction from SSP to finite-horizon MDP with Unknown $\costbound$}
    \label{alg:ssp-reduction-unknown-B}
    \begin{algorithmic}[1] 
    
        \STATE {\bfseries input:} state space $S$, action space $A$, initial state $\sinit$, goal state $\ssink$, confidence parameter $\delta$, number of episodes $K$, bound on the expected time of the optimal policy $\timebound$ and algorithm $\calA$ for regret minimization in finite-horizon MDPs.
        
        \STATE {\bfseries initialize} a \verb|Bernstein-SSP| algorithm $\calB_s$ with initial state $s$ and confidence parameter $\delta/|S|$ for every $s \in S$.
        
        \STATE set $L = 10^4 \timebound^2 |S|^2 |A| \log^3 \frac{K \timebound |S| |A|}{\delta}$, $S_\text{known}^1 = \{ \sinit \}$ and $N_f(s) = L \indevent{s = \sinit}$ for every $s \in S$.
        
        \STATE run $\calB_{\sinit}$ for $L$ episodes and set $\wt B_{\sinit}$ to be the average cost per episode times $10$.
        
          \STATE {\bfseries initialize} $\calA$ with state space $\wh S = S \cup \{ \ssink \}$, action space $A$, horizon $H = 8 \timebound \log (8K)$, confidence parameter $\frac{\delta}{4|S|}$, terminal costs $\hat c_f(s) = 8 \indevent{s = \sinit} \wt B_{\sinit}$ and bound on the expected cost of the optimal policy $9 \wt B_{\sinit}$.
        
        \STATE {\bfseries initialize} intervals counter $m \gets 0$, time steps counter $t \gets 1$ and epochs counter $e \gets 1$.
         
        \FOR{$k=L+1,\dots,K$}
            
            \STATE set $s_t \gets \sinit$.
            
            \WHILE{$s_t \neq \ssink$}
            
                \STATE set $m \gets m+1$, feed initial state $s_t$ to $\calA$ and obtain policy $\pi^m = \{ \pi^m_h: \wh S \to A \}_{h=1}^H$.
        
                \FOR{$h=1,\dots,H$}
        
                    \STATE play action $a_t = \pi^m_h(s_t)$, suffer cost $C_t \sim c(s_t,a_t)$, and set $s^m_h=s_t,a^m_h=a_t,C^m_h=C_t$.
            
                    \STATE observe next state $s_{t+1} \sim P(\cdot \mid s_t,a_t)$ and set $t \gets t+1$.
            
                    \IF{$s_t = \ssink$ or $s_t \not\in S_\text{known}^e$}
            
                        \STATE pad trajectory to be of length $H$ and BREAK.
            
                    \ENDIF
        
                \ENDFOR
                
                \STATE set $s^m_{H+1} = s_t$.
                
                \STATE feed trajectory $\traj{m} = (s^m_1,a^m_1,\dots,s^m_H,a^m_H,s^m_{H+1})$ and costs $\{ C^m_h \}_{h=1}^H$ to $\calA$.
                
                \IF{$s_t \not\in S_\text{known}^e$}
            
                        \STATE set $N_f(s_t) \gets N_f(s_t)+1$ and run an episode of $\calB_{s_t}$.
                        
                        \IF{$N_f(s_t) = L$}
                        
                        \STATE set $e \gets e+1$ and $S_\text{known}^e \gets S_\text{known}^{e-1} \cup \{ s_t \}$.
                        
                        \STATE set $\wt B_{s_t}$ to be the average cost per episode of $\calB_{s_t}$ times $10$.
                        
                        \STATE {\bfseries reinitialize} $\calA$ by updating the terminal costs as $\hat c_f(s) = 8 \indevent{s \in S_\text{known}^e} \max_{\tilde s \in S_\text{known}^e} \wt B_{\tilde s}$, updating the bound on the expected cost of the optimal policy $9 \max_{\tilde s \in S_\text{known}^e} \wt B_{\tilde s}$ and deleting the history of $\calA$ only in state $s_t$.
                        
                        \ENDIF
            
                    \ENDIF
        
            \ENDWHILE
        \ENDFOR
    \end{algorithmic}
\end{algorithm}

\subsection{Proof of Theorem~\ref{thm:regret-bound-with-admissible-algorithm-unknown-B}}

We follow the analysis of the known $\costbound$ case under the event that \cref{lem:estimate-B-with-Bernstein-SSP} holds for all states (which happens with probability at least $1 - \delta$), i.e., $\ctg{\piopt}(s) \le \wt B_s \le O(\costbound)$ for every $s \in S$.
We start by decomposing the regret similarly to \cref{lem:regret-to-finite-horizon-regret}.
Note that now there is an additional term that comes from the regret of the $|S|$ \verb|Bernstein-SSP| algorithms that are used to estimate $\costbound$.

\begin{lemma}
    \label{lem:regret-to-finite-horizon-regret-unknown-B}
    For $H = 8 \timebound \log (8K)$, we have the following bound on the regret of \cref{alg:ssp-reduction-unknown-B}:
    \begin{align}
        \label{eq:regret-decomposition-unknown-B}
        \regret 
        \le 
        \wh \calR_\calA(\numintervals) + \sum_{m=1}^\numintervals \brk*{\sum_{h=1}^H C^m_h + \hat c_f(s_{H+1}^m) - \hatctg{\pi^m}_1 (s^m_1)} + 
        O \left( \timebound^2 \costbound |S|^3 |A| \log^3 \frac{K \timebound |S| |A|}{\delta} \right),
    \end{align}
    where $\numintervals$ is the total number of intervals.
\end{lemma}

\begin{remark}
    Note that now each interval is considered in the context of the current epoch, i.e., the current $\costbound$-known states.
    The finite-horizon cost-to-go $\hatctg{\pi^m}$ is with respect to the MDP of $\costbound$-known states.
    Moreover, for interval $m$ that ends in a $\costbound$-unknown state, the last state in the trajectory $s_{H+1}^m$ will be a $\costbound$-unknown state and the length of the interval may be shorter than $H$ (just like intervals that end in the goal state).
\end{remark}

\begin{proof}
    Every interval ends either in the goal state, in a $\costbound$-known state or in a $\costbound$-unknown state.
    The first two cases are similar to the proof of \cref{lem:regret-to-finite-horizon-regret} because our estimates $\wt B_s$ in all $\costbound$-known states $s$ are upper bounds on $\ctg{\piopt}(s)$.
    Importantly, we do not initialize $\calA$ in the end of an epoch and this allows us to get its regret bound without an extra $|S|$ factor.
    The reason is that $\calA$ is an admissible (and thus optimistic) algorithm, so it operates based on the observations it collected.
    Another important note is that the cost in the virtual MDP $\calM^e$ is always bounded by the cost in the actual MDP $\calM$.
    
    We now focus on the last case.
    Recall that if interval $m$ ends in a $\costbound$-unknown state $s$, then the terminal cost is $0$ and we run an episode of the \verb|Bernstein-SSP| algorithm $\calB_s$.
    Thus, the excess cost of running \verb|Bernstein-SSP| algorithms is bounded by $|S|$ times the \verb|Bernstein-SSP| regret plus $|S| \costbound L$, i.e., we can bound it as follows
    \begin{align*}
        |S| \costbound L + O \left( \costbound^{3/2} |S|^2 \sqrt{|A| L} \log \frac{K \timebound |S| |A|}{\delta} + \timebound^{3/2} |S|^3 |A| \log^2 \frac{K \timebound |S| |A|}{\delta} \right).
    \end{align*}
    To finish the proof we plug in the definition of $L$.
\end{proof}

Next, we bound the number of intervals.
Again, we get a similar bound to \cref{lem:bound-on-number-of-intervals} but with an additional term for all the intervals that ended in a $\costbound$-unknown state (there are at most $|S|L$ such intervals).

\begin{lemma}
    \label{lem:bound-on-number-of-intervals-unknown-B}
    Assume that the reduction is performed using an admissible algorithm $\calA$.
    Then, with probability at least $1 - \nicefrac{3 \delta}{8}$,
    \[
        M 
        \le 
        4 \left( K + 10^4 |S| |A| \knownthresh \log \frac{K \timebound |S| |A| \knownthresh}{\delta} + 10^4 \timebound^2 |S|^3 |A| \log^3 \frac{K \timebound |S| |A|}{\delta} \right).
    \]
\end{lemma}

\begin{proof}
    The proof is based on the claim that in every interval there is a probability of at least $1/2$ that the agent reaches either the goal state, an unknown state-action pair or a $\costbound$-unknown state.
    This is proved similarly to \cref{lem:reach-goal-or-unknown-wp-1/2} since we can look at the MDP of $\costbound$-known states, and then the claim of \cref{lem:reach-goal-or-unknown-wp-1/2} is equivalent to reaching either the goal state, an unknown state-action pair or a $\costbound$-unknown state.
    
    With this claim the proof follows easily by following the proof of \cref{lem:bound-on-number-of-intervals}.
    We simply define $X^m$ to be $1$ if an unknown state-action pair or the goal or a $\costbound$-unknown state were reached during interval $m$ (and $0$ otherwise).
    Then, we have
    \[
        \sum_{m=1}^M X^m
        \le
        K + |S| |A| \knownthresh \log \frac{MH|S||A|}{\delta} + |S|L,
    \]
    which implies the Lemma following the same argument based on Freedman's inequality.
\end{proof}

Finally, we bound the deviation of the actual cost in each interval from its expected value.
The proof is exactly the same as \cref{lem:cost-deviation-from-value-function}.
The second moment of the accumulated cost until reaching the goal, an unknown state-action pair or a $\costbound$-unknown state is of order $\costbound^2$, and therefore in almost all intervals (except for a finite number) the accumulated cost will be of order $\costbound$ with high probability (in other intervals the cost is trivially bounded by $H + O(\costbound)$).

\begin{lemma}
    \label{lem:cost-deviation-from-value-function-unknown-B}
    Assume that the reduction is performed using an admissible algorithm $\calA$.
    Then, the following holds with probability at least $1 - \nicefrac{3 \delta}{8}$,
    \begin{align*}
        \sum_{m=1}^\numintervals \brk*{\sum_{h=1}^H C^m_h + \hat c_f(s_{H+1}^m) - \hatctg{\pi^m}_1(s_1^m)}
        & = 
        O \brk*{\costbound \sqrt{\numintervals \log \frac{M}{\delta}} + (H+\costbound) \knownthresh |S| |A| \log \frac{M K \timebound |S| |A|}{\delta}}
        \\
        & \qquad + 
        O \left( (H+\costbound) \timebound^2 |S|^3 |A| \log^3 \frac{K \timebound |S| |A|}{\delta} \right).
    \end{align*}
\end{lemma}

The proof of the theorem is finished by combining \cref{lem:regret-to-finite-horizon-regret-unknown-B,lem:bound-on-number-of-intervals-unknown-B,lem:cost-deviation-from-value-function-unknown-B} together with the guarantees of the admissible algorithm $\calA$ and \cref{lem:estimate-B-with-Bernstein-SSP}, similarly to \cref{thm:regret-bound-with-admissible-algorithm}.

\newpage

\section{Lower bound}
\label{sec:lower-bound-appendix}

In this section we prove \cref{thm:lowerbound} which lower bounds the expected regret of any learning algorithm for the case $\costbound < 1$.
It complements the lower bound found in \cite{rosenberg2020near} for the case $\costbound \ge 1$.

By Yao's minimax principle, in order to derive a lower bound on the learner's regret, it suffices to show a distribution over MDP instances that forces any deterministic learner to suffer a regret of $\Omega(\sqrt{\costbound |S| |A| K})$ in expectation.

To construct this distribution, we follow \cite{rosenberg2020near} with a few modifications. We initially consider the simpler setting with two states: an initial state and the goal state. We now embed a hard MAB instance into our problem where the optimal action has an expected cost of $\costbound$. 
To that end, consider a distribution over MDPs where a special action $a^\star$ is chosen a-priori uniformly at random. Then, all actions lead to the goal state $\ssink$ with probability 1. The cost $C_k(\sinit, a^\star)$ chosen at episode $k$ is 1 w.p.\ $\costbound$ and 0 otherwise. The cost of any other action $a \neq a^\star$ is 1 w.p.\ $\costbound + \epsilon$ and 0 otherwise, where $\epsilon \in (0, 1/8)$ is a constant to be determined. Thus the optimal policy will always play $a^\star$ and we have $\ctgopt(\sinit) = \costbound$.

Fix any deterministic learning algorithm, we shall now quantify the regret of the learner in terms of the number of times that it plays $a^\star$. Indeed, we have that the optimal cost is $\costbound$, and the learner loses $\epsilon$ in the regret each time she plays an action other than $a^\star$. Therefore,
\[
    \bbE \brk[s]{\regret} \ge \epsilon \cdot \brk{K - \bbE \brk[s]{N}},
\]
where $N$ is the number of times $a^\star$ was chosen in $\sinit$.

We now introduce an additional distribution of the costs which denote by $\Pr_\text{unif}$.
$\Pr_\text{unif}$ is identical to the distribution over the costs defined above, and denoted by $\Pr$, except that $\Pr[C_k(\sinit, a) = 1] = \costbound + \epsilon$ for all actions $a \in A$ regardless of the choice of $a^\star$.
We denote expectations over $\Pr_\text{unif}$ by $\bbE_\text{unif}$, and expectations over $\Pr$ by $\bbE$.
The following lemma uses standard lower bound techniques used for multi-armed bandits (see, e.g., \citealp[Theorem 13]{jaksch2010near}) to bound the difference in the expectation of $N$ 
when the learner plays in $\Pr$ compared to when it plays in $\Pr_\text{unif}$.

\begin{lemma}
    \label{lem:astarub}
    Suppose that $\costbound \le \frac12$.
    Denote by $\Pr_{\text{unif},a}$, $\bbE_{\text{unif},a}$, $\Pr_a$, $\bbE_a$ the distributions and expectations defined above conditioned on $a^\star = a$.
    For any deterministic learner we have that    
    $
        \bbE_a [N]
        \le
        \bbE_{\text{unif},a} [N]
        +
        \epsilon K \sqrt{ \bbE_{\text{unif},a} [N]/ \costbound}.
    $
\end{lemma}

\begin{proof}
    Fix any deterministic learner.
    Let us denote by $C^{(k)}$ the sequence of costs observed by the learner up to episode $k$ and including.
    Now, as $N \le K$ and the fact that $N$ is a deterministic function of $C^{(K)}$,
    $
        \bbE_a [N]
        \le
        \bbE_{\text{unif},a} [N]
        + 
        K \cdot \TV{\Pr_{\text{unif},a}[C^{(K)}]}{\Pr[C^{(K)}]},
    $
    and Pinsker's inequality yields
    \begin{equation}
        \label{eq:pinsker}
        \TV{\Pr_{\text{unif},a}[C^{(K)}]}{\Pr[C^{(K)}]}
        \le 
        \sqrt{\frac{1}{2} \KL{\Pr_{\text{unif},a}[C^{(K)}]}{\Pr_a[C^{(K)}]}}.
    \end{equation}
    Next, the chain rule of the KL divergence obtains
    \begin{align*}
        &\KL{\Pr_{\text{unif},a}[C^{(K)}]}{\Pr_a[C^{(K)}]} \\
        &\qquad =
        \sum_{k=1}^K \sum_{C^{(k)}} \Pr_{\text{unif},a}[C^{(k)}] \cdot \KL{\Pr_{\text{unif},a}[C_k(\sinit, a_k) \mid C^{(k)}]}{\Pr_a[C_k(\sinit, a_k) \mid C^{(k)}]},
    \end{align*}
    where $a_k$ is the action chosen by the learner at episode $k$. (Recall that after which the model transition to the goal state and the episode ends.)
    
    Observe that at any episode, since the learning algorithm is deterministic, the learner chooses an action given $C^{(k)}$ regardless of whether $C^{(k)}$ was generated under $\Pr$ or under $\Pr_{\text{unif},a}$. Thus, the $\KL{\Pr_{\text{unif},a}[C_k(\sinit, a_k) \mid C^{(k)}]}{\Pr_a[C_k(\sinit, a_k) \mid C^{(k)}]}$ is zero if $a_k \neq a_\star$, and otherwise
    \begin{align*}
        &\KL{\Pr_{\text{unif},a}[C_k(\sinit, a_k) \mid C^{(k)}]}{\Pr_a[C_k(\sinit, a_k) \mid C^{(k)}]} \\
        &\qquad =
        (\costbound + \epsilon) \log \brk2{1 + \frac{\epsilon}{\costbound}} + (1 - \costbound - \epsilon) \log \brk2{1 - \frac{\epsilon}{1 - \costbound}} \\
        &\qquad \le
        \frac{\epsilon^2}{\costbound ( 1- \costbound)},
    \end{align*}
    where we used that $\log(1+x) \le x$ for all $x > -1$, and since we assume $\costbound \le \frac12$ and $\epsilon < \frac18$ that imply $-\epsilon/(1-\costbound) \ge -\frac14 > -1$.
    Plugging the above back into \cref{eq:pinsker} and using $\costbound \le \frac12$ gives the lemma.
\end{proof}

In the following result, we combine the lemma above with standard techniques from lower bounds of multi-armed bandits (see \citealp{auer2002nonstochastic} for example).

\begin{theorem}
    \label{thm:twostatelb}
    Suppose that $\costbound \le \frac12$, $\epsilon \in (0,\frac18)$ and $|A| \ge 2$.
    For the problem described above we have that
    \[
        \bbE[\regret] \ge \epsilon K 
        \brk3{\frac{1}{2} - \epsilon \sqrt{\frac{K}{|A| \costbound}}}.
    \]
\end{theorem}

\begin{proof}[Proof of \cref{thm:twostatelb}]
    Note that as under $\Pr_\text{unif}$ the cost distributions of all actions are identical. Denote by $N_a$ the number of times that the learner chooses action $a$ in $\sinit$. Therefore,
    \begin{equation}
        \label{eq:unifsumactions}
        \sum_{a \in A} \bbE_{\text{unif},a}  [N] 
        =
        \sum_{a \in A} \bbE_\text{unif}  [N_a] 
        = 
        \bbE_\text{unif}  \brk[s]*{\sum_{a \in A} N_a}
        = 
        K.
    \end{equation}
    
    Recall that $a^\star$ is sampled uniformly at random before the game starts. 
    Then,
    \begin{align*}
        \bbE[\regret]
        &=
        \frac{1}{|A|} \sum_{a \in A} \bbE_a[\regret] \\
        &\ge
        K - \frac{1}{|A|} \sum_{a \in A} \bbE_a[N]
        \\
        &\ge
        K - \frac{1}{|A|} \sum_{a \in A} \brk2{\bbE_{\text{unif},a} [N]
        +
        \epsilon K \sqrt{ \bbE_{\text{unif},a} [N]/ \costbound}}
        \tag{\cref{lem:astarub}} \\
        &\ge
        K - \frac{1}{|A|} \sum_{a \in A} \bbE_{\text{unif},a} [N]
        +
        \epsilon K \sqrt{ \frac{1}{|A| \costbound} \sum_{a \in A}\bbE_{\text{unif},a} [N]}
        \tag{Jensen's inequality} \\
        &=
        K - \frac{K}{|A|} 
        +
        \epsilon K \sqrt{ \frac{K}{|A| \costbound}},
        \tag{\cref{eq:unifsumactions}}
    \end{align*}
    The theorem follows from $|A| \ge 2$ and by rearranging.
\end{proof}

\begin{proof}[Proof of \cref{thm:lowerbound}]
    Consider the following MDP. 
    Let $S$ be the set of states disregarding $\ssink$. 
    The initial state is sampled uniformly at random from $S$. 
    Each $s \in S$ has its own special action $a^\star_s$. 
    All actions transition to the goal state with probability 1. 
    The cost $C_k(s,a)$ of action $a \neq a^\star_s$ in episode $k$ and state $s$ is 1 with probability $\costbound + \epsilon$ and 0 otherwise. The cost of $C_k(s,a^\star_s)$ is 1 with probability $\costbound$ and 0 otherwise.
    
    Note that for each $s \in S$, the learner is faced with a simple problem as the one described above from which it cannot learn about from other states $s' \neq s$. 
    Therefore, we can apply \cref{thm:twostatelb} for each $s \in S$ separately and lower bound the learner's expected regret the sum of the regrets suffered at each $s \in S$, which would depend on the number of times $s \in S$ is drawn as the initial state. Since the states are chosen uniformly at random there are many states (constant fraction) that are chosen $\Theta(K/|S|)$ times. Summing the regret bounds of \cref{thm:twostatelb} over only these states and choosing $\epsilon$ appropriately gives the sought-after bound.
    
    Denote by $K_s$ the number of episodes that start in each state $s \in S$. 
    \begin{align}
        \bbE[\regret ]
        \ge 
        \sum_{s \in S} 
        \bbE
        \brk[s]3{\epsilon K_s \brk2{\frac{1}{2} - \epsilon \sqrt{\frac{K_s}{|A| \costbound}} }}
        =
        \frac{\epsilon K}{2} - \epsilon^2 \sqrt{\frac{1}{|A| \costbound}} \sum_{s \in S} \bbE[K_s^{3/2}]. 
        \label{eq:regretlb}
    \end{align}
    Applying Cauchy-Schwartz inequality gives
    \begin{align*}
        \sum_{s \in S} \bbE [K_s^{3/2}]
        &\le
        \sum_{s \in S} \sqrt{\bbE [K_s]} \sqrt{\bbE [K_s^2]}
        =
        \sum_{s \in S} \sqrt{\bbE [K_s]} \sqrt{\bbE [K_s]^2 + \text{Var}[K_s]}
        \\
        &=
        \sum_{s \in S} \sqrt{\frac{K}{|S|}} \sqrt{\frac{K^2}{|S|^2} + \frac{K }{|S|}\brk2{1-\frac{1}{|S|}}}
        \le
        K \sqrt{\frac{2K}{|S|}},
    \end{align*}
    where we have used the expectation and variance formulas of the Binomial distribution.
    The lower bound is now given by applying the inequality above in \cref{eq:regretlb} and choosing $\epsilon = \frac{1}{8} \sqrt{\costbound |A| |S|/K}$.
\end{proof}

\newpage

\section{General useful results}

\begin{lemma}[Freedman's Inequality]
    \label{lemma: freedmans inequality}
    Let $\brc{X_t}_{t\geq 1}$ be a real valued martingale difference sequence adapted to a filtration $\brc*{F_t}_{t\geq 0}$. 
    If $|X_t| \leq R$ a.s. then for any $\eta \in (0,1/R), T \in \mathbb{N}$ it holds with probability at least $1-\delta$,
    \[
        \sum_{t=1}^T X_t 
        \leq 
        \eta \sum_{t=1}^T \E[X_t^2| F_{t-1}] + \frac{\log(1/\delta)}{\eta}.
    \]
\end{lemma}

\begin{lemma}[Consequences of Freedman's Inequality for Bounded and Positive Sequence of Random Variables, e.g.,~\citealp{efroni2021confidence}, Lemma 27]
    \label{lemma: consequences of optimism and freedman's inequality}
    Let $\brc{Y_t}_{t\geq 1}$ be a real valued sequence of random variables adapted to a filtration $\brc*{F_t}_{t\geq 0}$. 
    Assume that for all $t\geq 1$ it holds that $0\leq Y_{t}\leq C$ a.s., and $T\in \mathbb{N}$.
    Then, each of the following inequalities hold with probability at least $1-\delta$.
    \begin{align*}
        \sum_{t=1}^T \E[Y_t|F_{t-1}]
        & \leq 
        \br*{1+\frac{1}{2C}} \sum_{t=1}^T Y_t + 2(2C+1)^2 \log\frac{1}{\delta}
        \\
        \sum_{t=1}^T Y_t 
        & \leq 
        2\sum_{t=1}^T \E[Y_t|F_{t-1}] + 4C\log\frac{1}{\delta}.
    \end{align*}
\end{lemma}

\begin{lemma}[Standard Deviation Difference, e.g., \citealp{zanette2019tighter}]
    \label{lemma: std difference}
    Let $V_1,V_2: S \rightarrow \mathbb{R}$ be fixed mappings.
    Let $P(s)$ be a probability measure over the state space.
    Then, $\sqrt{\VAR(V_1)} - \sqrt{\VAR(V_2)} \leq \sqrt{\VAR(V_1-V_2)}$.
\end{lemma}

\end{document}